\newif\ifdraft \drafttrue
\newif\iffull \fulltrue

\makeatletter \@input{tex.flags} \makeatother

\documentclass[11pt]{article}

\usepackage{subfigure}
\usepackage{algorithm}
\usepackage[noend]{algorithmic}
\usepackage{kpfonts}
\usepackage{fullpage}
\usepackage{amsmath, amssymb, amsthm}
\usepackage{bbm}
\usepackage{mathtools}
\usepackage{color}
\usepackage{xcolor}
\usepackage{authblk}

\usepackage{multicol}

\definecolor{DarkGreen}{rgb}{0.1,0.5,0.1}
\definecolor{DarkRed}{rgb}{0.5,0.1,0.1}
\definecolor{DarkBlue}{rgb}{0.1,0.1,0.5}
\usepackage[]{hyperref}
\hypersetup{
    unicode=false,          
    pdftoolbar=true,        
    pdfmenubar=true,        
    pdffitwindow=false,      
    pdftitle={},    
    pdfauthor={}
    pdfsubject={},   
    pdfnewwindow=true,      
    pdfkeywords={keywords}, 
    colorlinks=true,       
    linkcolor=DarkRed,          
    citecolor=DarkGreen,        
    filecolor=DarkRed,      
    urlcolor=DarkBlue,          
}
\usepackage{xspace}
\usepackage{cleveref}
\usepackage{thmtools, thm-restate}
\usepackage{natbib}

\newcommand{\sw}[1]{\ifdraft \textcolor{blue}{[Steven: #1]}\fi}

\newcommand\RR{\mathbb{R}}
\newcommand\cA{\mathcal{A}}
\newcommand\cC{\mathcal{C}}

\newcommand\cL{\mathcal{L}}
\newcommand\cF{\mathcal{G}}
\newcommand\cH{\mathcal{H}}
\newcommand\cG{\mathcal{G}}
\newcommand\cT{\mathcal{T}}

\newcommand\cP{\mathcal{P}}

\newcommand\cX{\mathcal{X}}
\newcommand\cD{\mathcal{D}}

\newcommand\cS{\mathcal{S}}

\newcommand\FFP{{\bf{FairFictPlay\;}}}
\newcommand\FP{\mathrm{FP}}
\newcommand\SP{\mathrm{SP}}
\newcommand\PA{P_\mathrm{audit}}

\newcommand\PYD{P_{y=0}^D}
\newcommand\PDO{P_{D=1}^y}
\newcommand\PXD{P^D}
\newcommand\VCD{\mathrm{VCDIM}}

\newcommand\LC{\mathrm{LC}}
\newcommand\CSC{\mathrm{CSC}}
\newcommand\spsize{\alpha_{SP}}
\newcommand\spdisp{\beta_{SP}}
\newcommand\fpsize{\alpha_{FP}}
\newcommand\fpdisp{\beta_{FP}}

\DeclareMathOperator{\poly}{poly}

\DeclareMathOperator*{\Expectation}{\mathbb{E}}
\newcommand{\Ex}[2]{\Expectation_{#1}\left[#2\right]}

\newcommand{\eps}{\varepsilon}
\def\epsilon{\varepsilon}
\DeclareMathOperator{\OPT}{OPT}

\DeclareMathOperator*{\argmin}{\mathrm{argmin}}
\DeclareMathOperator*{\argmax}{\mathrm{argmax}}
\newcommand{\INDSTATE}[1][1]{\STATE\hspace{#1\algorithmicindent}}

\newtheorem{theorem}{Theorem}[section]
\newtheorem{lemma}[theorem]{Lemma}

\newtheorem{claim}[theorem]{Claim}
\newtheorem{remark}[theorem]{Remark}
\newtheorem{corollary}[theorem]{Corollary}

\newtheorem{assumpt}[theorem]{Assumption}
\newtheorem{definition}[theorem]{Definition}
\newtheorem{example}[theorem]{Example}

\title{Preventing Fairness Gerrymandering: \\ Auditing and Learning for Subgroup Fairness}

\author[1]{Michael Kearns}
\author[1]{Seth Neel}
\author[1]{Aaron Roth}
\author[2]{Zhiwei Steven Wu}
\affil[1]{University of Pennsylvania}
\affil[2]{Microsoft Research-New York City}

\begin{document}

\maketitle

\begin{abstract}
  The most prevalent notions of fairness in machine learning are
  \emph{statistical} definitions: they fix a small collection of
  high-level, pre-defined groups (such as race or gender), and then
  ask for approximate parity of some statistic of the classifier (like
  positive classification rate or false positive rate) across these
  groups.  Constraints of this form are susceptible to (intentional or
  inadvertent) \emph{fairness gerrymandering}, in which a classifier
  appears to be fair on each individual group, but badly violates the
  fairness constraint on one or more structured \emph{subgroups}
  defined over the protected attributes (such as certain combinations
  of protected attribute values).  We propose instead to demand
  statistical notions of fairness across exponentially (or infinitely)
  many subgroups, defined by a structured class of functions over the
  protected attributes. This interpolates between statistical
  definitions of fairness, and recently proposed individual notions of
  fairness, but it raises several computational challenges. It is no
  longer clear how to even check or \emph{audit} a fixed classifier to
  see if it satisfies such a strong definition of fairness. We prove
  that the computational problem of auditing subgroup fairness for
  both equality of false positive rates and statistical parity is
  equivalent to the problem of weak agnostic learning --- which means
  it is computationally hard in the worst case, even for simple
  structured subclasses. However, it also suggests that common
  heuristics for learning can be applied to successfully solve the
  auditing problem in practice.

We then derive two algorithms that provably converge to the best fair
distribution over classifiers in a given class, given access to
oracles which can optimally solve the agnostic learning problem.  The
algorithms are based on a formulation of subgroup fairness as a
two-player zero-sum game between a Learner (the primal player) and an
Auditor (the dual player). Both algorithms compute an equilibrium of
this game. We obtain our first algorithm by simulating play of the
game by having Learner play an instance of the no-regret {\em Follow
  the Perturbed Leader\/} algorithm, and having Auditor play best
response. This algorithm provably converges to an approximate Nash
equilibrium (and thus to an approximately optimal subgroup-fair
distribution over classifiers) in a polynomial number of steps.  We
obtain our second algorithm by simulating play of the game by having
both players play {\em Fictitious Play\/}, which enjoys only provably
asymptotic convergence, but has the merit of simplicity and faster
per-step computation.  We implement the Fictitious Play version using
linear regression as a heuristic oracle, and show that we can
effectively both audit and learn fair classifiers on real datasets.
\end{abstract}

\thispagestyle{empty} \setcounter{page}{0}
\clearpage

\newcommand{\E}{\mathbb E}

\section{Introduction}

As machine learning is being deployed in increasingly consequential domains (including policing \citep{policing}, criminal sentencing \citep{sentencing}, and lending \citep{credit}), the problem of ensuring that learned models are {\em fair\/} has become urgent.

Approaches to fairness in machine learning can coarsely be divided into two kinds: \emph{statistical} and \emph{individual} notions of fairness. Statistical notions typically fix a
small number of protected demographic groups $\mathcal{G}$ (such as racial groups), and then ask for (approximate)
parity of some statistical measure across all of these groups.
One popular statistical measure asks for equality of false positive or negative rates across all groups in $\mathcal{G}$ (this is also sometimes referred to as an {\em equal opportunity\/} constraint \citep{HPS}). Another asks for equality of classification rates (also known as \emph{statistical parity}).
These statistical notions of fairness are the kinds of fairness definitions most common in the literature (see e.g. \cite{kamiran2012data,hajian2013methodology,KMR16,HPS,FSV16,zafar2017fairness,Chou17}).

One main attraction of statistical definitions of fairness is that they can in principle be obtained and checked without making any assumptions about the underlying population, and hence lead to more immediately actionable algorithmic approaches. On the other hand, individual notions of fairness ask for the algorithm to satisfy some guarantee which binds at the individual, rather than group, level. This often has the semantics that ``individuals who are similar'' should be treated ``similarly'' \citep{DworkFairness}, or ``less qualified individuals should not be favored over more qualified individuals'' \citep{JKMR16}. Individual notions of fairness have attractively strong semantics, but their main drawback is that achieving them seemingly requires more assumptions to be made about the setting under consideration.

The semantics of statistical notions of fairness would be significantly stronger if they were defined over a large number of {\em subgroups\/},
thus permitting a rich middle ground between fairness only for a small number of coarse pre-defined groups, and
the strong assumptions needed for fairness at the individual level.
Consider the kind of {\em fairness gerrymandering\/} that can occur when we only look for unfairness over a small number of pre-defined groups:
\begin{example}
\label{example}
Imagine a setting with two binary features, corresponding to race (say black and white) and gender (say male and female), both of which are distributed independently and uniformly at random in a population. Consider a classifier that labels an example positive if and only if it corresponds to a black man, or a white woman. Then the classifier will appear to be equitable when one considers either protected attribute alone, in the sense that it labels both men and women as positive 50\% of the time, and labels both black and white individuals as positive 50\% of the time. But if one looks at any \emph{conjunction} of the two attributes (such as black women), then it is apparent that the classifier maximally violates the statistical parity fairness constraint. Similarly, if examples have a binary label that is also distributed uniformly at random, and independently from the features, the classifier will satisfy equal opportunity fairness with respect to either protected attribute alone, even though it maximally violates it with respect to conjunctions of two attributes.
\end{example}

We remark that the issue raised by this toy example is not merely hypothetical.
In our experiments in Section \ref{sec:empirical}, we show that similar violations of fairness on
subgroups of the pre-defined groups can result from the application of standard machine learning
methods applied to real datasets.
To avoid such problems, we would like to be able to satisfy a fairness constraint not just for
the small number of protected groups defined by single protected attributes, but for a combinatorially large or even infinite
collection of structured subgroups definable over protected attributes.


In this paper, we consider the problem of {\em auditing\/} binary classifiers for equal opportunity and statistical parity,
and the problem of {\em learning\/}  classifiers subject to these constraints,
when the number of protected groups is large.
There are exponentially many ways of carving up a population into subgroups, and we cannot necessarily identify a small number of these \emph{a priori}
as the only ones we need to be concerned about. At the same time, we cannot insist on any notion of statistical fairness for \emph{every} subgroup of the population: for example, any imperfect classifier could be accused of being unfair to the subgroup of individuals defined ex-post as the set of individuals it misclassified. This simply corresponds to ``overfitting'' a fairness constraint. We note that the individual fairness definition of \cite{JKMR16} (when restricted to the binary classification setting) can be viewed as asking for equalized false positive rates across the singleton subgroups, containing just one individual each\footnote{It also asks for equalized false negative rates, and that the false positive rate is smaller than the true positive rate. Here, the randomness in the ``rates'' is taken entirely over the randomness of the classifier.} --- but naturally, in order to achieve this strong definition of fairness, \cite{JKMR16} have to make structural assumptions about the form of the ground truth. It is, however, sensible to ask for fairness for large \emph{structured} subsets of individuals: so long as these subsets have a bounded VC dimension, the \emph{statistical} problem of learning and auditing fair classifiers is easy, so long as the dataset is sufficiently large. This can be viewed as an interpolation between equal opportunity fairness and the individual ``weakly meritocratic'' fairness definition from \cite{JKMR16}, that does not require making any assumptions about the ground truth.
Our investigation focuses on the computational challenges, both in theory and in practice.

\subsection{Our Results}

Briefly, our contributions are:
\begin{itemize}
	\item Formalization of the problem of auditing and learning classifiers	
		for fairness with respect to rich classes of subgroups $\mathcal{G}$.
	\item Results proving (under certain assumptions)
		the computational equivalence of auditing
		$\mathcal{G}$ and (weak) agnostic learning of $\mathcal{G}$.
		While these results imply theoretical intractability of auditing for some
		natural classes $\mathcal{G}$, they also suggest that practical machine
		learning heuristics can be applied to the auditing problem.
	\item Provably convergent algorithms for learning classifiers
		that are fair with respect to $\mathcal{G}$, based on a formulation
		as a two-player zero-sum game between a Learner (the primal player) and
		an Auditor (the dual player). We provide two different algorithms, both of which are based on solving for the equilibrium of this game.
		The first provably converges in a polynomial number of steps and is based on simulation of the game dynamics when the Learner
		uses {\em Follow the Perturbed Leader\/} and the Auditor uses best response; the second
		is only guaranteed to converge asymptotically but is computationally simpler, and involves both players using
		{\em Fictitious Play\/}. 
	\item An implementation and empirical evaluation of the Fictitious Play algorithm demonstrating its
		effectiveness on a real dataset in which subgroup fairness is a concern.
		
\end{itemize}

In more detail,
we start by studying the computational challenge of simply \emph{checking} whether a given classifier satisfies equal opportunity and statistical parity. Doing this in time linear in the number of protected groups is simple: for each protected group, we need only estimate a single expectation. However, when there are many different protected attributes which can be combined to define the protected groups, their number is combinatorially large\footnote{For example, as discussed in a recent Propublica investigation \citep{propublicafacebook}, Facebook policy protects groups against hate speech if the group is definable as a \emph{conjunction} of protected attributes.  Under the Facebook schema, ``race'' and ``gender'' are both protected attributes, and so the Facebook policy protects ``black women'' as a distinct class, separately from black people and women.  When there are $d$ protected attributes, there are $2^d$ protected groups. As a statistical estimation problem, this is not a large obstacle --- we can estimate $2^d$ expectations to error $\epsilon$ so long as our data set has size $O(d/\epsilon^2)$, but there is now a computational problem.}.

We model the problem by specifying a class of functions $\cG$ defined over a set of $d$ protected attributes. $\cG$ defines a set of protected subgroups.
Each function $g \in \cG$ corresponds to the protected
subgroup $\{x : g_i(x) = 1\}$\footnote{For example, in the case of Facebook's policy, the protected attributes include ``race, sex, gender identity, religious affiliation, national origin, ethnicity, sexual orientation and serious disability/disease'' \citep{propublicafacebook}, and $\cG$ represents the class of boolean conjunctions. In other words, a group defined by individuals having any \emph{subset} of values for the protected attributes is protected.}. The first result of this paper is that for both equal opportunity and statistical parity, the computational problem of \emph{checking} whether a classifier or decision-making algorithm $D$ violates statistical fairness with respect to the set of protected groups $\cG$ is equivalent to the problem of \emph{agnostically learning} $\cG$ \citep{agnostic}, in a strong and
distribution-specific sense. This equivalence has two implications:
\begin{enumerate}
\item First, it allows us to import \emph{computational hardness} results from the learning theory literature. Agnostic learning turns out to be computationally hard in the worst case, even for extremely simple classes of functions $\cG$ (like boolean conjunctions and linear threshold functions). As a result, we can conclude that auditing a classifier $D$ for statistical fairness violations with respect to a class $\cG$ is also computationally hard. This means we should not expect to find a polynomial time algorithm that is always guaranteed to solve the auditing problem.

\item However, in practice, various learning heuristics (like boosting, logistic regression, SVMs, backpropagation for neural networks, etc.) are commonly used to learn accurate classifiers which are known to be hard to learn in the worst case. The equivalence we show between agnostic learning and auditing is \emph{distribution specific} --- that is, if on a particular data set, a heuristic learning algorithm can solve the agnostic learning problem (on an appropriately defined subset of the data), it can be used also to solve the auditing problem on the same data set. 
\end{enumerate}
These results appear in Section \ref{sec:auditing}.

Next, we consider the problem of \emph{learning} a classifier that equalizes false positive or negative
rates across all (possibly infinitely many) sub-groups, defined by a class of functions $\cG$. As per the reductions
described above, this problem is computationally hard in the worst case.

However, under the assumption that we have an efficient oracles which solves the \emph{agnostic learning} problem,
we give and analyze algorithms for this problem based on a game-theoretic formulation.
We first prove that the optimal fair classifier can be found as the equilibrium of a two-player, zero-sum game,
in which the (pure) strategy space of the ``Learner'' player corresponds to classifiers in $\cH$,
and the (pure) strategy space of the ``Auditor'' player corresponds to subgroups defined by $\cG$.
The best response problems for the two players correspond to agnostic learning and auditing, respectively. We show that both problems can be solved with a single call to a \emph{cost sensitive classification oracle}, which is equivalent to an agnostic learning oracle.
We then draw on extant theory for learning in games and no-regret algorithms to derive two different algorithms based on simulating game play in this formulation.
In the first, the Learner employs the well-studied {\em Follow the Perturbed Leader (FTPL)\/} algorithm on an appropriate linearization of
its best-response problem, while the Auditor approximately best-responds to the distribution over classifiers of the Learner at each step.
Since FTPL has a no-regret guarantee, we obtain an algorithm that provably converges in a polynomial number of steps.  

While it enjoys strong provable guarantees, this first algorithm is randomized (due to the noise added by FTPL), and the best-response step for the
Auditor is polynomial time but computationally expensive. We thus propose a second algorithm that is deterministic, simpler and faster per step, based on both players adopting the Fictitious Play learning dynamic. This algorithm has weaker theoretical guarantees: it has provable convergence only asymptotically, and not in a polynomial number of steps --- but is more practical and
converges rapidly in practice.
The derivation of these algorithms (and their guarantees) appear in Section \ref{sec:learning}. 

Finally, we implement the Fictitious Play algorithm and demonstrate its practicality by efficiently learning classifiers that approximately equalize false positive rates across any group definable by a linear threshold function on 18 protected attributes in the ``Communities and Crime'' dataset. We use simple, fast regression algorithms as heuristics to implement agnostic learning oracles, and (via our reduction from agnostic learning to auditing) auditing oracles. Our results suggest that it is possible in practice to learn fair classifiers with respect to a large class of subgroups that still achieve non-trivial error. Full details are contained in Section \ref{sec:empirical}, and for a substantially more comprehensive empirical investigation of our method we direct the interested reader to \cite{experimental}.

\subsection{Further Related Work}

Independent of our work, \cite{calibration} also consider a related and complementary notion of fairness that they call ``multicalibration''. In settings in which one wishes to train a real-valued predictor, multicalibration can be considered the ``calibration'' analogue for the definitions of subgroup fairness that we give for false positive rates, false negative rates, and classification rates. For a real-valued predictor, calibration informally requires that for every value $v \in [0,1]$ predicted by an algorithm, the fraction of individuals who truly have a positive label in the subset of individuals on which the algorithm predicted $v$ should be approximately equal to $v$. Multicalibration asks for approximate calibration on every set defined implicitly by some circuit in a set $\cG$. \cite{calibration} give an algorithmic result that is analogous to the one we give for learning subgroup fair classifiers: a polynomial time algorithm for learning a multi-calibrated predictor, given an agnostic learning algorithm for $\cG$. In addition to giving a polynomial-time algorithm, we also give a practical variant of our algorithm (which is however only guaranteed to converge in the limit) that we use to conduct empirical experiments on real data.

Thematically, the most closely related piece of prior work is \cite{heuristicauditing}, who also aim to audit classification algorithms for discrimination in subgroups that have not been pre-defined. Our work differs from theirs in a number of important ways. First, we audit the algorithm for common measures of statistical unfairness, whereas \cite{heuristicauditing} design a new measure compatible with their particular algorithmic technique. Second, we give a formal analysis of our algorithm. Finally, we audit with respect to subgroups defined by a class of functions $\cG$, which we can take to have bounded VC dimension, which allows us to give formal out-of-sample guarantees.  \cite{heuristicauditing} attempt to audit with respect to \emph{all possible} sub-groups, which introduces a severe multiple-hypothesis testing problem, and risks overfitting. Most importantly we give actionable algorithms for learning subgroup fair classifiers, whereas \cite{heuristicauditing} restrict attention to auditing.

Technically, the most closely related piece of work (and from which we take inspiration for our algorithm in Section \ref{sec:learning}) is \cite{reductions}, who show that given access to an agnostic learning oracle for a class $\cH$, there is an efficient algorithm to find the lowest-error distribution over classifiers in $\cH$ subject to equalizing false positive rates across polynomially many subgroups. Their algorithm can be viewed as solving the same zero-sum game that we solve, but in which the ``subgroup'' player plays gradient descent over his pure strategies, one for each sub-group. This ceases to be an efficient or practical algorithm when the number of subgroups is large, as is our case. Our main insight is that an agnostic learning oracle is sufficient to have the both players play ``fictitious play'', and that there is a transformation of the best response problem such that an agnostic learning algorithm is enough to efficiently implement follow the perturbed leader. 

There is also other work showing computational hardness for fair learning problems. Most notably, \cite{srebro} show that finding a linear threshold classifier that approximately minimizes hinge loss subject to equalizing false positive rates across populations is computationally hard (assuming that refuting a random $k$-XOR formula is hard). In contrast, we show that even \emph{checking} whether a classifier satisfies a false positive rate constraint on a particular data set is computationally hard (if the number of subgroups on which fairness is desired is too large to enumerate).


\section{Model and Preliminaries}

We model each individual as being described by a tuple $((x, x'), y)$,
where $x\in \cX$ denotes a vector of \emph{protected attributes},
$x'\in \cX'$ denotes a vector of \emph{unprotected attributes}, and
$y\in \{0, 1\}$ denotes a label. Note that in our formulation, an
auditing algorithm not only may not see the unprotected attributes
$x'$, it may not even be aware of their existence. For example, $x'$
may represent proprietary features or consumer data purchased by a
credit scoring company.

We will write $X = (x, x')$ to denote the joint feature vector. We
assume that points $(X, y)$ are drawn i.i.d. from an unknown
distribution $\mathcal{P}$.  Let $D$ be a decision making algorithm,
and let $D(X)$ denote the (possibly randomized) decision induced by
$D$ on individual $(X,y)$. We restrict attention in this paper to the
case in which $D$ makes a binary classification decision:
$D(X) \in \{0,1\}$. Thus we alternately refer to $D$ as a classifier.
When \emph{auditing} a fixed classifier $D$, it will be helpful to
make reference to the distribution over examples $(X, y)$ together
with their induced classification $D(X)$. Let $\PA(D)$ denote the
induced \emph{target joint distribution} over the tuple $(x, y, D(X))$
that results from sampling $(x, x', y) \sim \mathcal{P}$, and
providing $x$, the true label $y$, and the classification
$D(X) = D(x,x')$ but not the unprotected attributes $x'$.  Note that
the randomness here is over both the randomness of $\mathcal{P}$, and
the potential randomness of the classifier $D$.

We will be concerned with learning and auditing classifiers $D$
satisfying two common statistical fairness constraints: equality of
classification rates (also known as statistical parity), and equality
of false positive rates (also known as equal opportunity). Auditing
for equality of false negative rates is symmetric and so we do not
explicitly consider it. Each fairness constraint is defined with
respect to a set of protected groups. We define sets of protected
groups via a family of indicator functions $\cF$ for those groups,
defined over protected attributes. Each $g:\cX \to \{0,1\} \in \cF$
has the semantics that $g(x) = 1$ indicates that an individual with
protected features $x$ is in group $g$.

\begin{definition}[Statistical Parity (SP) Subgroup Fairness]
  Fix any classifier $D$, distribution $\mathcal{P}$, collection of
  group indicators $\cG$, and parameter $\gamma\in [0, 1]$.
  For each $g\in \cG$, define
  \begin{align*}
    \spsize(g, \cP) = \Pr_{\cP}[g(x) = 1] \quad \mbox{and,} \quad
    \spdisp(g, D, \cP) = \left|\SP(D) - \SP(D, g)\right|,
  \end{align*}
  where $\SP(D) = \Pr_{\cP, D}[D(X) = 1]$ and
  $\SP(D, g) = \Pr_{\cP, D}[D(X) = 1 | g(x) = 1]$ denote the overall
  acceptance rate of $D$ and the acceptance rate of $D$ on group $g$
  respectively.  We say that $D$ satisfies $\gamma$-\emph{statistical
    parity (SP) Fairness} with respect to $\mathcal{P}$ and $\cF$ if
  for every $g \in \cF$
  \[ \spsize(g, \cP) \; \spdisp(g, D, \cP)\leq \gamma.\] We will
  sometimes refer to $\SP(D)$ as the \emph{SP base rate}.
\end{definition}

\begin{remark}
  Note that our definition references two approximation parameters,
  both of which are important. We are allowed to ignore a 
  group $g$ if it (or its complement) represent only a small fraction
  of the total probability mass. The parameter $\alpha$ governs how
  small a fraction of the population we are allowed to
  ignore. Similarly, we do not require that the probability of a
  positive classification in every subgroup is exactly equal to the
  base rate, but instead allow deviations up to $\beta$. Both of these
  approximation parameters are necessary from a statistical estimation
  perspective. We control both of them with a single parameter
  $\gamma$.
\end{remark}

\begin{definition}[False Positive (FP) Subgroup Fairness]\label{fp-fair}
  Fix any classifier $D$, distribution $\mathcal{P}$, collection of
  group indicators $\cF$, and parameter $\gamma \in [0,1]$.  For each
  $g\in \cG$, define
  \begin{align*}
    \fpsize(g, \cP) = \Pr_{\cP}[g(x) = 1, y=0] \quad \mbox{and,} \quad
                      \fpdisp(g, D, \cP) = \left| \FP(D) - \FP(D, g) \right|
  \end{align*}
  where $\FP(D) = \Pr_{D, \cP} [D(X) = 1\mid y=0]$ and
  $\FP(D, g) = \Pr_{D, \cP}[D(X) = 1\mid g(x) =1 , y=0]$ denote the
  overall false-positive rate of $D$ and the false-positive rate of
  $D$ on group $g$ respectively.

  We say $D$
  satisfies $\gamma$-\emph{False Positive (FP) Fairness} with respect
  to $\mathcal{P}$ and $\cF$ if for every $g \in \cG$
  \[ \fpsize(g, \cP) \; \fpdisp(g, D, \cP) \leq \gamma.\] We will
  sometimes refer to $\FP(D)$ FP-base rate.
\end{definition}

\begin{remark}
  This definition is symmetric to the definition of statistical parity
  fairness, except that the parameter $\alpha$ is now used to exclude
  any group $g$ such that \emph{negative examples} ($y = 0$) from
  $g$ (or its complement) have probability mass less than
  $\alpha$. This is again necessary from a statistical estimation perspective.
\end{remark}

For either statistical parity and false positive fairness, if the
algorithm $D$ fails to satisfy the $\gamma$-fairness condition, then
we say that $D$ is $\gamma$-\emph{unfair} with respect to
$\mathcal{P}$ and $\cF$. We call any subgroup $g$ which witnesses this
unfairness an $\gamma$-\emph{unfair certificate} for
$(D, \mathcal{P})$.

An \emph{auditing algorithm} for a notion of fairness is given sample access to $\PA(D)$ for some classifier $D$. It will either deem $D$ to be fair with respect to $\mathcal{P}$, or will else produce a certificate of unfairness.

\begin{definition}[Auditing Algorithm]
  Fix a notion of fairness (either statistical parity or false
  positive fairness), a collection of group indicators $\cF$ over the
  protected features, and any $\delta, \gamma, \gamma' \in (0, 1)$
  such that $\gamma'\leq \gamma$.  A {$(\gamma, \gamma')$-auditing
    algorithm for $\cF$ with respect to distribution $\mathcal{P}$} is
  an algorithm $\cA$ such that for any classifier $D$, when given
  access the distribution $\PA(D)$, $\cA$ runs in time
  $\poly(1/\gamma', \log(1/\delta))$, and with probability $1-\delta$,
  outputs a $\gamma'$-unfair certificate for $D$ whenever $D$ is
  $\gamma$-unfair with respect to $\mathcal{P}$ and $\cF$. If $D$ is
  $\gamma'$-fair, $\cA$ will output ``fair''.
\end{definition}

As we will show, our definition of auditing is closely related to  weak
agnostic learning.

\begin{definition}[Weak Agnostic Learning~\citep{agnostic,KMV08}]
  Let $Q$ be a distribution over $\cX\times\{0, 1\}$ and let
  $\eps, \eps' \in (0, 1/2)$ such that $\eps \geq \eps'$. We say that
  the function class $\cF$ is \emph{$(\eps , \eps')$-weakly
    agnostically learnable} under distribution $Q$ if there exists an
  algorithm $L$ such that when given sample access to $Q$, $L$ runs in
  time $\poly(1/\eps', 1/\delta)$, and with probability $1 - \delta$,
  outputs a hypothesis $h\in \cF$ such that
  \[
    \min_{f\in\cF} err(f, Q) \leq 1/2 - \eps \implies err(h, Q) \leq
    1/2 - \eps'.
  \]
  where $err(h,Q) = \Pr_{(x,y) \sim Q}[h(x) \neq y]$.
\end{definition}

\paragraph{Cost-Sensitive Classification.}
In this paper, we will also give reductions to \emph{cost-sensitive classification
  (CSC)} problems. Formally, an instance of a CSC problem for the
class $\cH$ is given by a set of $n$ tuples
$\{(X_i, c_i^0, c_i^1)\}_{i=1}^n$ such that $c_i^\ell$ corresponds to
the cost for predicting label $\ell$ on point $X_i$. Given such an
instance as input, a CSC oracle finds a hypothesis $\hat h \in \cH$ that
minimizes the total cost across all points:
\begin{equation}
  \hat h \in \argmin_{h\in \cH} \sum_{i = 1}^n [h(X_i) c_i^1 + (1 -
  h(X_i)) c_i^0]\label{csc}
\end{equation}
A crucial property of a CSC problem is that the solution is invariant
to translations of the costs.

\begin{claim}
\label{claim:csc}
  Let $\{(X_i, c_i^0, c_i^1)\}_{i=1}^n$ be a CSC instance, and
  $\{(\tilde c_i^0, \tilde c_i^1)\}$ be a set of new costs such that
  there exist $a_1, a_2,\ldots, a_n\in \RR$ such that
  $\tilde c_i^\ell = c_i^\ell + a_i$ for all $i$ and $\ell$. Then
\[
  \argmin_{h\in \cH} \sum_{i = 1}^n [h(X_i) c_i^1 + (1 - h(X_i))
  c_i^0] = \argmin_{h\in \cH} \sum_{i = 1}^n [h(X_i) \tilde c_i^1 + (1
  - h(X_i)) \tilde c_i^0]
\]
\end{claim}

\begin{remark}
  We note that cost-sensitive classification is polynomially
  equivalent to agnostic learning~\cite{ZLA03}.  We give both
  definitions above because when describing our results for auditing,
  we wish to directly appeal to known hardness results for weak
  agnostic learning, but it is more convenient to describe our
  algorithms via oracles for cost-sensitive classification.
\end{remark} 

\paragraph{Follow the Perturbed Leader.}

We will make use of the \emph{Follow the Perturbed Leader (FTPL)}
algorithm as a no-regret learner for online linear optimization
problems \citep{KV05}.  To formalize the algorithm, consider
$\cS \subset \{0, 1\}^d$ to be a set of ``actions'' for a learner  in an online decision problem. The
learner interacts with an adversary over $T$ rounds, and in each round
$t$, the learner (randomly) chooses some action $a^t \in \cS$, and the
adversary chooses a loss vector $\ell^t \in [-M , M]^d$. The learner incurs a loss of $\langle \ell^t, a^t \rangle$ at round $t$.

FTPL is a simple algorithm that in each round perturbs the cumulative
loss vector over the previous rounds
$\overline \ell = \sum_{s< t} \ell^s$, and chooses the action that
minimizes loss with respect to the perturbed cumulative loss vector. We present the full
algorithm in \Cref{alg:ftpl}, and its formal guarantee in
\Cref{thm:ftpl-regret}.

\begin{algorithm}[h]
  \caption{Follow the Perturbed Leader (FTPL) Algorithm}
 \label{alg:ftpl}
  \begin{algorithmic}
    \STATE{\textbf{Input}: Loss bound $M$, action set
      $\cS \in \{0, 1\}^d$ } \STATE{\textbf{Initialize}: Let
      $\eta = (1/M) \sqrt{\frac{1}{\sqrt{d} T}} $, $\cD_U$ be the
      uniform distribution over $[0, 1]^d$, and let $a^1\in \cS$ be
      arbitrary.}

    \STATE{\textbf{For} $t = 1, \ldots, T$:}
    \INDSTATE{Play action
      $a^t$; Observe loss vector $\ell^t$ and suffer loss
      $\langle \ell^t , a^t \rangle$.}

    \INDSTATE{Update:}
    \[
      a^{t+1} = \argmin_{a\in \cS} \left[ \eta \sum_{s \leq t} \langle
        \ell^s , a \rangle + \langle \xi^t , a \rangle\right]
    \]
    \INDSTATE{where $\xi^t$ is drawn independently for each $t$ from the distribution $\cD_U$.  }
    \end{algorithmic}
  \end{algorithm}

  \begin{theorem}[\cite{KV05}]\label{thm:ftpl-regret}
    For any sequence of loss vectors $\ell^1, \ldots , \ell^T$, the
    FTPL algorithm has regret
    \[
      \Ex{}{\sum_{t=1}^T \langle \ell^t , a^t \rangle } - \min_{a\in
        \cS} \sum_{t=1}^T \langle \ell^t, a \rangle \leq 2 d^{5/4}
      M\sqrt{T}
    \]
    where the randomness is taken over the perturbations $\xi^t$ across
    rounds.
  \end{theorem}


\subsection{Generalization Error}
\label{subsec:gen}
In this section, we observe that the error rate of a classifier $D$,
as well as the degree to which it violates $\gamma$-fairness (for both
statistical parity and false positive rates) can be accurately
approximated with the empirical estimates for these quantities on a
dataset (drawn i.i.d. from the underlying distribution $\mathcal{P}$)
so long as the dataset is sufficiently large. Once we establish this
fact, since our main interest is in the computational problem of
auditing and learning, in the rest of the paper, we assume that we
have direct access to the underlying distribution (or equivalently,
that the empirical data defines the distribution of interest), and do
not make further reference to sample complexity or overfitting issues.

A standard VC dimension bound (see, e.g. \cite{KV94}) states:
\begin{theorem}
  Fix a class of functions $\cH$. For any distribution $\mathcal{P}$,
  let $S \sim \mathcal{P}^m$ be a dataset consisting of $m$ examples
  $(X_i, y_i)$ sampled i.i.d. from $\mathcal{P}$. Then for any
  $0 < \delta < 1$, with probability $1-\delta$, for every
  $h \in \cH$, we have:
$$\left|err(h, \mathcal{P}) - err(h, S) \right| \leq O\left(\sqrt{\frac{\mathrm{VCDIM}(\cH)\log m + \log(1/\delta)}{m}} \right)$$
where
$err(h, S) = \frac{1}{m}\sum_{i=1}^m \mathbbm{1}[h(X_i) \neq y_i]$.
\end{theorem}

The above theorem implies that so long as $m \geq \tilde O(\mathrm{VCDIM}(\cH)/\epsilon^2)$, then minimizing error over the empirical sample $S$ suffices to minimize error up to an additive $\epsilon$ term on the true distribution $\mathcal{P}$. Below, we give two analogous statements for fairness constraints:
\begin{theorem}[SP Uniform Convergence]
\label{thm:SP-uc}
Fix a class of functions $\cH$ and a class of group indicators
$\cG$. For any distribution $\mathcal{P}$, let $S \sim \mathcal{P}^m$
be a dataset consisting of $m$ examples $(X_i, y_i)$ sampled
i.i.d. from $\mathcal{P}$. Then for any $0 < \delta < 1$, with
probability $1-\delta$, for every $h \in \cH$ and $g \in \cG$
\[
  \left| \spsize(g, \cP_S) \; \spdisp(g, h, \cP_S) - {\spsize(g, \cP) \;
      \spdisp(g, h, \cP)}\right| \leq
  \tilde{O}\left(\sqrt{\frac{(\mathrm{VCDIM}(\cH) +
        \mathrm{VCDIM}(\cG))\log m + \log(1/\delta)}{m}} \right)
\]
where $\cP_S$ denotes the empirical distribution over the realized
sample $S$.
\end{theorem}

Similarly:
\begin{theorem}[FP Uniform Convergence]
\label{thm:FP-uc}
Fix a class of functions $\cH$ and a class of group indicators
$\cG$. For any distribution $\mathcal{P}$, let $S \sim \mathcal{P}^m$
be a dataset consisting of $m$ examples $(X_i, y_i)$ sampled
i.i.d. from $\mathcal{P}$. Then for any $0 < \delta < 1$, with
probability $1-\delta$, for every $h \in \cH$ and $g \in \cG$, we
have:
$$
\left| \fpsize(g, \cP) \; \fpdisp(g, D, \cP) - {\fpsize(g, \cP) \;
    \fpdisp(g, D, \cP)}\right| \leq
\tilde{O}\left(\sqrt{\frac{(\mathrm{VCDIM}(\cH) +
      \mathrm{VCDIM}(\cG))\log m + \log(1/\delta)}{m}} \right)$$ where
$\cP_S$ denotes the empirical distribution over the realized sample
$S$.
\end{theorem}

These theorems together imply that for both SP and FP subgroup
fairness, the degree to which a group $g$ violates the constraint of
$\gamma$-fairness can be estimated up to error $\epsilon$, so long as
$m \geq \tilde O((\mathrm{VCDIM}(\cH) +
\mathrm{VCDIM}(\cG))/\epsilon^2)$. The proofs can be found in Appendix
\ref{app:generalization}.



\section{Equivalence of Auditing and Weak Agnostic Learning}
\label{sec:auditing}

In this section, we give a reduction from the problem of auditing both statistical parity and false positive rate fairness, to the problem of agnostic learning, and vice versa. This has two implications. The main implication is that, from a worst-case analysis point of view, auditing is computationally hard in almost every case (since it inherits this pessimistic state of affairs from agnostic learning).
However, worst-case hardness results in learning theory have not prevented the successful practice of machine learning,
and there are many heuristic algorithms that in real-world cases successfully solve ``hard'' agnostic learning problems. 
Our reductions also imply that these heuristics can be used successfully as auditing algorithms, and we exploit this in the
development of our algorithmic results and their experimental evaluation.


We make the following mild assumption on the class of group indicators $\cG$, to aid in our reductions. It is satisfied by most natural classes of functions, but is in any case essentially without loss of generality (since learning negated functions can be simulated by learning the original function class on a dataset with flipped class labels).

\begin{assumpt}
  We assume the set of group indicators $\cG$ satisfies closure under
  negation: for any $g\in \cG$, we also have $\neg g\in \cG$.
\end{assumpt}

Recalling that $X = (x,x')$ and the following notions will be useful
for describing our results:
\begin{itemize}
\item $\SP(D) = \Pr_{\cP, D}[D(X) = 1]$ and
  $\FP(D) = \Pr_{D, \cP} [D(X) = 1\mid y=0]$.
\item $\spsize(g, \cP) = \Pr_{\cP}[g(x) = 1]$ and
  $\fpsize(g, \cP) = \Pr_{\cP}[g(x) = 1, y=0]$.
\item $\spdisp(g, D, \cP) = \left|\SP(D) - \SP(D, g)\right|$ and
  $\fpdisp(g, D, \cP) = \left| \FP(D) - \FP(D, g) \right|$.

\item $\PXD$: the marginal distribution on $(x, D(X))$.
\item $\PYD$: the conditional distribution on $(x, D(X))$, conditioned on $y = 0$.
\end{itemize}
We will think about these as the target distributions for a learning problem: i.e. the problem of learning to predict $D(X)$ from only the protected features $x$. We will relate the ability to agnostically learn on these distributions, to the ability to audit $D$ given access to the original distribution $\PA(D)$.

\subsection{Statistical Parity Fairness}

We give our reduction first for SP subgroup fairness. The reduction
for FP subgroup fairness will follow as a corollary, since auditing
for FP subgroup fairness can be viewed as auditing for
statistical parity fairness on the subset of the data restricted to
$y = 0$.

\begin{theorem}\label{thm:reduction1}
  Fix any distribution $\cP$, and any set of
  group indicators $\cG$. 
Then for any
  $\gamma, \epsilon > 0$, the following relationships hold:
\begin{itemize}
\item If there is a $(\gamma/2, (\gamma/2 - \eps))$ auditing algorithm
  for $\cG$ for all $D$ such that $\SP(D) = 1/2$, then the class $\cF$
  is $(\gamma, \gamma/2 - \eps)$-weakly agnostically learnable under
  $\PXD$.

\item If $\cG$ is $(\gamma, \gamma - \eps)$-weakly agnostically
  learnable under distribution $\PXD$ for all $D$ such that
  $\SP(D) = 1/2$, then there is a $(\gamma, (\gamma - \eps)/2)$
  auditing algorithm for $\cG$ for SP fairness under $\cP$.
\end{itemize}
\end{theorem}

We will prove \Cref{thm:reduction1} in two steps. First, we show that
any unfair certificate $f$ for $D$ has non-trivial error for
predicting the decision made by $D$ from the sensitive attributes.

\begin{lemma}\label{dis-acc}
  Suppose that the base rate $\SP(D) \leq 1/2$ and there exists a
  function $f$ such that
  $$\spsize(g, \cP) \; \spdisp(g, D, \cP)= \gamma.$$  Then
  $$
  \max\{\Pr[D(X) = f(x)], \Pr[D(X) = \neg f(x)]\} \geq \SP(D) +
  \gamma.
  $$
\end{lemma}

\begin{proof}
  To simplify notations, let $b = \SP(D)$ denote the base rate,
  $\alpha = \spsize$ and $\beta = \spdisp$. First, observe that either
  $\Pr[D(X) = 1 \mid f(x) = 1] = b+ \beta$ or
  $\Pr[D(X) = 1\mid f(x) = 1] = b - \beta$ holds.

  In the first case, we know $\Pr[D(X) = 1\mid f(x) = 0] < b$, and
  so $\Pr[D(X) = 0 \mid f(x) = 0] > 1 - b$. It follows that
  \begin{align*}
    \Pr[D(X) = f(x) ] &= \Pr[D(X) = f(x) = 1] +\Pr[D(X) =
      f(x) = 0]\\
    &= \Pr[D(X) = 1 \mid f(x) = 1]\Pr[f(x) = 1] + \Pr[D(X) = 0 \mid f(x) = 0] \Pr[f(x)=0]\\
    &> \alpha (b + \beta) + (1-\alpha) (1 - b) \\ 
    &=  (\alpha - 1) b + (1 - \alpha) (1 - b) + b + \alpha\beta \\
    &= (1 - \alpha)(1 - 2b) + b + \alpha\beta.
  \end{align*}
  In the second case, we have
  $\Pr[D(X) = 0 \mid f(x) = 1] = (1 - b) + \beta$ and
  $\Pr[D(X) = 1\mid f(x) = 0] > b$. We can then bound
  \begin{align*}
   \Pr[D(X) = f(x) ] &= \Pr[D(X) = 1 \mid f(x) = 0]\Pr[f(x) = 0] + \Pr[D(X) = 0 \mid f(x) = 1] \Pr[f(x)=1]\\
    &> (1- \alpha) b  + \alpha (1 - b + \beta) =  \alpha(1 - 2b) +
      b+ \alpha\beta.
  \end{align*}
  In both cases, we have $(1 - 2b) \geq 0$ by our assumption on the
  base rate. Since $\alpha\in [0, 1]$, we know
  $$
  \max\{\Pr[D(X) = f(x)], \Pr[D(X) = \neg f(x)]\} \geq
  b + \alpha\beta = b+\gamma
  $$
  which recovers our bound.
\end{proof}

In the next step, we show that if there exists any function $f$ that
accurately predicts the decisions made by the algorithm $D$, then
either $f$ or $\neg f$ can serve as an unfairness certificate for $D$.

\begin{lemma}\label{acc-dis}
  Suppose that the base rate $\SP(D) \geq 1/2$ and there exists
  a function $f$ such that
  $\Pr[D(X) = f(x)] \geq {\SP}(D) + \gamma$ for some value
  $\gamma \in (0, 1/2)$. Then there exists a function $g$ such that
 $$\spsize(g, \cP) \; \spdisp(g, D, \cP) \geq \gamma/2,$$ 
 where $g \in \{f, \neg f\}$.
\end{lemma}

\begin{proof}
  Let $b = {\SP}(D)$.  We can expand $\Pr[D(X) = f(x)]$ as
  follows:
  \begin{align*}
    \Pr[D(X) = f(x) ] &= \Pr[D(X) = f(x) = 1] +\Pr[D(X) =
                        f(x) = 0]\\
                      &= \Pr[D(X) = 1 \mid f(x) = 1]\Pr[f(x) = 1] + \Pr[D(X) = 0 \mid f(x) = 0] \Pr[f(x)=0]
  \end{align*}
  \noindent This means
  \begin{align*}
    &   \Pr[D(X) = f(x) ] - b \\= &\left(\Pr[D(X) = 1 \mid f(x) = 1] - b\right)\Pr[f(x) = 1] + \left(\Pr[D(X) = 0 \mid f(x) = 0] - b \right)\Pr[f(x)=0] \geq \gamma
  \end{align*}
  Suppose that
  $\left(\Pr[D(X) = 1 \mid f(x) = 1] - b\right)\Pr[f(x) = 1] \geq
  \gamma / 2$, then our claim holds with $g = f$.
  Suppose not, then we must have
  \begin{align*}
    \left(\Pr[D(X) = 0 \mid f(x) = 0] - b \right)\Pr[f(x)=0] &=  \left((1 - b) -\Pr[D(X) = 1 \mid f(x) = 0] \right)\Pr[f(x)=0]\geq \gamma/2
  \end{align*}
  Note that by our assumption $b\geq (1-b)$. This means
  \begin{align*}
    \left(b -\Pr[D(X) = 1 \mid f(x) = 0]
    \right)\Pr[f(x)=0]\geq     \left((1 - b) -\Pr[D(X) = 1 \mid f(x) = 0]
    \right)\Pr[f(x)=0]\geq \gamma/2
\end{align*}
which implies that our claim holds with $g = \neg f$.  
\end{proof}

\begin{proof}[Proof of Theorem~\ref{thm:reduction1}]
  Suppose that the class $\cG$ satisfies
  $\min_{f\in \cG} err(f, \PXD) \leq 1/2 - \gamma$. Then by
  \Cref{acc-dis}, there exists some $g\in \cG$ such that
  $\Pr[g(x) = 1]|\Pr[D(X) = 1 \mid g(x) = 1] - \SP(D)| \geq
  \gamma/2$.  By the assumption of auditability, we can then use the
  auditing algorithm to find a group $g'\in \cG$ that is an
  $(\gamma/2 - \eps)$-unfair certificate of $D$.  By \Cref{dis-acc},
  we know that either $g'$ or $\neg g'$ predicts $D$ with an accuracy
  of at least $1/2 + (\gamma/2 - \eps)$.

  In the reverse direction, consider the auditing problem on the
  classifier $D$.  We can treat each pair $(x, D(X))$ as a labelled
  example and learn a hypothesis in $\cG$ that approximates the
  decisions made by $D$.  Suppose that $D$ is $\gamma$-unfair. Then by
  \Cref{dis-acc}, we know that there exists some $g\in \cG$ such that
  $\Pr[D(X) = g(x)] \geq 1/2 + \gamma$. Therefore, the weak agnostic
  learning algorithm from the hypothesis of the theorem will return
  some $g'$ with $\Pr[D(X) = g'(x)] \geq 1/2 + (\gamma - \eps)$. By
  \Cref{acc-dis}, we know $g'$ or $\neg g'$ is a
  $(\gamma - \eps)/2$-unfair certificate for $D$.
\end{proof}

\subsection{False Positive Fairness}

A corollary of the above reduction is an analogous equivalence between
auditing for FP subgroup fairness and agnostic learning. This is
because a FP fairness constraint can be viewed as a statistical parity
fairness constraint on the subset of the data such that $y =
0$. Therefore, \Cref{thm:reduction1} implies the following:

\begin{corollary}
\label{cor:reduction2}
  Fix any distribution $\cP$, and any set of group indicators $\cG$. 
  The following
  two relationships hold:
\begin{itemize}
\item If there is a $(\gamma/2, (\gamma/2 - \eps))$ auditing algorithm
  for $\cF$ for all $D$ such that ${\FP}(D) = 1/2$,
  then the class $\cG$ is
  $(\gamma, \gamma/2 - \eps)$-weakly agnostically learnable under $\PYD$.
\item If $\cG$ is $(\gamma, \gamma - \eps)$--weakly agnostically
  learnable under distribution $\PYD$ for all $D$ such that
  ${\FP}(D) = 1/2$, then there is
  a $(\gamma, (\gamma - \eps)/2)$ auditing algorithm for FP subgroup
  fairness for $\cG$ under distribution
  $\cP$. 
\end{itemize}
\end{corollary}

\subsection{Worst-Case Intractability of Auditing}

While we shall see in subsequent sections that the equivalence given above
has positive algorithmic and experimental consequences, from a purely theoretical
perspective the reduction of agnostic
learning to auditing has strong negative worst-case implications. More precisely,
we can import a long sequence of formal intractability results for agnostic learning
to obtain:

\begin{theorem}
\label{thm:audithard}
Under standard complexity-theoretic intractability assumptions,
for $\cG$ the classes of conjunctions of boolean attributes,
linear threshold functions, or
bounded-degree polynomial threshold functions, there exist distributions $P$
such that the auditing problem cannot be solved in polynomial time,
for either statistical parity or false positive fairness.
\end{theorem}

The proof of this theorem follows from Theorem~\ref{thm:reduction1}, Corollary~\ref{cor:reduction2},
and the following negative
results from the learning theory literature.
\cite{FGRW12} show a strong negative result for weak agnostic learning
for conjunctions: given a distribution on labeled examples from the
hypercube such that there exists a monomial (or conjunction)
consistent with $(1 - \eps)$-fraction of the examples, it is NP-hard
to find a halfspace that is correct on $( 1/ 2 + \eps)$-fraction of
the examples, for arbitrary constant $\eps > 0$.
\cite{DOSW11} show that under the Unique Games Conjecture, no
polynomial-time algorithm can find a degree-$d$ polynomial threshold
function (PTF) that is consistent with $(1/2 + \eps)$ fraction of a
given set of labeled examples, even if there exists a degree-$d$ PTF
that is consistent with a $(1 - \eps)$ fraction of the examples.
\cite{DOSW11} also show that it is NP-Hard to find a degree-2 PTF that
is consistent with a $(1/2 + \eps)$ fraction of a given set of labeled
examples, even if there exists a halfspace (degree-1 PTF) that is
consistent with a $(1 - \eps)$ fraction of the examples.

While Theorem~\ref{thm:audithard} shows that certain natural subgroup classes $\cG$ yield
intractable auditing problems in the worst case, in the rest of the paper we demonstrate
that effective heuristics for this problem on specific (non-worst case) distributions can be used to derive
an effective and practical learning algorithm for subgroup fairness.


\section{A Learning Algorithm Subject to Fairness Constraints $\cG$}
\label{sec:learning}

In this section, we present an algorithm for training a (randomized)
classifier that satisfies false-positive subgroup fairness
simultaneously for all protected subgroups specified by a family of
group indicator functions $\cG$. All of our techniques also apply to a
statistical parity or false negative rate constraint.

Let $S$ denote a set of $n$ labeled examples
$\{z_i = (x_i, x'_i), y_i)\}_{i=1}^n$, and let $\cP$ denote the
empirical distribution over this set of examples.  Let $\cH$ be a
hypothesis class defined over both the protected and unprotected
attributes, and let $\cG$ be a collection of group indicators over the
protected attributes. We assume that $\cH$ contains a constant
classifier (which implies that there is at least one fair classifier
to be found, for any distribution).

Our goal will be to find the distribution over classifiers from $\cH$
that minimizes classification error subject to the fairness constraint
over $\cG$. We will design an iterative algorithm that, when given
access to a CSC oracle, computes an optimal randomized classifier in
polynomial time.

Let $D$ denote a probability distribution over $\cH$. Consider the
following {\em Fair ERM (Empirical Risk Minimization)} problem:
\begin{align}
  &  \min_{D\in \Delta_\cH}\; \Ex{h\sim D}{err(h, \cP)}\\
  \mbox{such that } \forall g\in \cG \qquad &
\fpsize(g, \cP) \; \fpdisp(g, D, \cP) \leq \gamma.
\end{align}
where $err(h, \cP) = \Pr_\cP[h(x, x') \neq y]$, and the quantities
$\fpsize$ and $\fpdisp$ are defined in \Cref{fp-fair}.
We will write $\OPT$ to denote the objective value at the optimum for
the Fair ERM problem, that is the minimum error achieved by a
$\gamma$-fair distribution over the class $\cH$.

Observe that the optimization is feasible for any distribution $\cP$:
the constant classifiers that labels all points 1 or 0 satisfy all
subgroup fairness constraints. At the moment, the number of decision
variables and constraints may be infinite (if $\cH$ and $\cG$ are infinite hypothesis classes),
but we will address this momentarily.

\begin{assumpt}[Cost-Sensitive Classification Oracle]
  We assume our algorithm has access to the cost-sensitive
  classication oracles $\CSC(\cH)$ and $\CSC(\cG)$ over the classes
  $\cH$ and $\cG$.
\end{assumpt}

Our main theoretical result is an computationally efficient
oracle-based algorithm for solving the Fair ERM problem.

\begin{restatable}{theorem}{polytime}\label{thm:polytime}
  Fix any $\nu, \delta\in (0, 1)$. Then given an input of $n$ data
  points and accuracy parameters $\nu, \delta$ and access to oracles
  $\CSC(\cH)$ and $\CSC(\cG)$, there exists an algorithm runs in
  polynomial time, and with probability at least $1 - \delta$, output
  a randomized classifier $\hat D$ such that
  $err(\hat D, \cP) \leq \OPT + \nu$, and for any $g\in \cG$, the
  fairness constraint violations satisfies
  \[
    \alpha_{FP}(g, \cP) \; \beta_{FP}(g, \hat D, \cP) \leq \gamma + O(\nu).
  \]
\end{restatable}

\paragraph{Overview of our solution.} We present our solution in steps:
\begin{itemize}
\item
{\bf Step 1: Fair ERM as LP.}
First, we rewrite the Fair ERM problem as a linear program with
finitely many decision variables and constraints even when $\cH$ and $\cG$ are infinite. To do
this, we take advantage of the fact that Sauer's Lemma lets us bound
the number of labellings that any hypothesis class $\cH$ of bounded
VC dimension can induce on any fixed dataset. The LP has one variable
for each of these possible labellings, rather than one
variable for each hypothesis. Moreover, again by Sauer's Lemma,
we have one constraint for each
of the finitely many possible subgroups induced by $\cG$ on the fixed
dataset, rather than one for each of the (possibly infinitely many)
subgroups definable over arbitrary datasets. This step is important
--- it will guarantee that strong duality holds.

\item
{\bf Step 2: Formulation as Game.}
We then derive the partial Lagrangian of the LP, and note that
computing an approximately optimal solution to this LP is equivalent
to finding an approximate minmax solution for a corresponding zero-sum
game, in which the payoff function $U$ is the value of the
Lagrangian. The pure strategies of the primal or ``Learner'' player
correspond to classifiers $h \in \cH$, and the pure strategies of the
dual or ``Auditor'' player correspond to subgroups $g \in \cG$.
Intuitively, the Learner is trying to minimize the sum of the prediction
error and a fairness penalty term (given by the Lagrangian), and the
Auditor is trying to penalize the fairness violation of the Learner by
first identifying the subgroup with the greatest fairness violation and
putting all the weight on the dual variable corresponding to this
subgroup. 
In order to reason about convergence, we restrict the set of dual
variables to lie in a bounded set: $C$ times the probability
simplex. $C$ is a parameter that we have to set in the proof of our
theorem to give the best theoretical guarantees --- but it is also a
parameter that we will vary in the experimental section.

\item
{\bf Step 3: Best Responses as CSC.}
We observe that given a mixed strategy for the Auditor, the best response problem of the Learner corresponds to a CSC problem. Similarly, given a mixed strategy for the Learner, the best response problem of the Auditor corresponds to an auditing problem (which can be represented as a CSC problem). Hence, if we have oracles for solving CSC problems, we can compute best responses for both players, in response to arbitrary mixed strategies of their opponents.

\item
{\bf Step 4: FTPL for No-Regret.}
Finally, we show that the ability to compute best responses for each
player is sufficient to implement dynamics known to converge quickly to
equilibrium in zero-sum games.  Our algorithm has the Learner
play {\em Follow the Perturbed Leader (FTPL)\/}~\cite{KV05}, which is
a no-regret algorithm, against an Auditor who at every round best
responds to the learner's mixed strategy. By the seminal result of
\citet{FS96}, the average plays of both players converge to an
approximate equilibrium.  In order to implement this in polynomial
time, we need to represent the loss of the learner as a
low-dimensional linear optimization problem. To do so, we first define
an appropriately translated CSC problem for any mixed strategy
$\lambda$ by the Auditor, and cast it as a linear optimization problem.
\end{itemize}

\subsection{Rewriting the Fair ERM Problem}
To rewrite the Fair ERM problem, we note that even though both $\cG$
and $\cH$ can be infinite sets, the sets of possible labellings on the
data set $S$ induced by these classes are finite. More formally, we
will write $\cG(S)$ and $\cH(S)$ to denote the set of all labellings
on $S$ that are induced by $\cG$ and $\cH$ respectively, that is
\[
  \cG(S) = \{(g(x_1), \ldots , g(x_n)) \mid g\in \cG\} \qquad
  \mbox{and,} \qquad \cH(S) = \{(h(X_1), \ldots , h(X_n))\mid h\in
  \cH\}
\]

We can bound the cardinalities of $\cG(S)$ and $\cH(S)$ using
Sauer's Lemma.

\begin{lemma}[Sauer's Lemma (see e.g. \cite{KV94})]\label{lem:sauer}
  Let $S$ be a data set of size $n$.  Let $d_1 = \VCD(\cH)$ and
  $d_2 = \VCD(\cG)$ be the VC-dimensions of the two classes. Then
  \[
    |\cH(S)| \leq O\left(n^{d_1} \right) \qquad \mbox{ and } \qquad
    |\cG(S)| \leq O\left(n^{d_2} \right).
  \]
\end{lemma}

Given this observation, we can then consider an equivalent
optimization problem where the distribution $D$ is over the set of
labellings in $\cH(S)$, and the set of subgroups are defined by the
labellings in $\cG(S)$. We will view each $g$ in $\cG(S)$ as a Boolean
function.

To simplify notations, we will define the following ``fairness
violation'' functions for any $g\in \cG$ and any $h\in \cH$:
  \begin{align}
    &\Phi_+(h, g) \equiv \fpsize(g, P) \, \left(\FP(h) - \FP(h, g) \right) - \gamma\\
    & \Phi_-(h, g) \equiv \fpsize(g, \cP)\, \left(\FP(h, g) - \FP(h)\right) - \gamma
  \end{align}
Moreover, for any distribution $D$ over $\cH$, for any sign $\bullet \in \{+,-\}$
\[
  \Phi_\bullet(D, g) = \Ex{h\sim D}{\Phi_\bullet(h, g)}.
\]
\begin{claim}\label{cl:stuff}
  For any $g\in \cG$, $h\in \cH$, and any $\nu > 0$,
\[
  \max\{\Phi_+(D, g), \Phi_-(D,g)\}\leq \nu \quad\mbox{ if and only if }\quad
  \alpha_{FP}(g, \cP)\; \beta_{FP}(g, D, \cP) \leq \gamma + \nu.
\]
\end{claim}

Thus, we
will focus on the following equivalent optimization problem.
\begin{align}
    \min_{D\in \Delta_{\cH(S)}}& \Ex{h\sim D}{err(h, \cP)}   \\
\mbox{such that for each }  g\in \cG(S):\qquad
  & \Phi_+(D, g)\leq 0 \label{f1}\\
  & \Phi_-(D,g) \leq 0 \label{f2}
\end{align}

For each pair of constraints \eqref{f1} and \eqref{f2}, corresponding
to a group $g\in \cG(S)$, we introduce a pair of dual variables
$\lambda_g^+$ and $\lambda_g^-$. The partial Lagrangian of the linear
program is the following:
\begin{align*}
  \cL(D, \lambda) = \Ex{h\sim D}{err(h, \cP)} +  \sum_{g\in \cG(S)}
  \left( \lambda_g^+\,  \Phi_+(D, g) +  \lambda_g^- \,  \Phi_-(D,g) \right)
\end{align*}

By Sion's minmax theorem~\citep{sion1958}, we have
\[
  \min_{D\in \Delta_{\cH(S)}}\, \max_{\lambda\in \RR_+^{2|\cG(S)|}} \cL(p,
  \lambda) = \max_{\lambda\in \RR_+^{2|\cG(S)|}}\, \min_{D\in
    \Delta_{\cH(S)}} \cL(p, \lambda) = \OPT
\]
where $\OPT$ denotes the optimal objective value in the fair ERM
problem. Similarly, the distribution
$\arg\min_{D}\, \max_{\lambda} \cL(D, \lambda)$ corresponds to an
optimal feasible solution to the fair ERM linear program.  Thus,
finding an optimal solution for the fair ERM problem reduces to
computing a minmax solution for the Lagrangian. Our algorithms will
both compute such a minmax solution by iteratively optimizing over
both the primal variables $D$ and dual variables $\lambda$. In order
to guarantee convergence in our optimization, we will restrict the
dual space to the following bounded set:
\[
  \Lambda = \{\lambda \in \RR_+^{2|\cG(S)|} \mid \|\lambda\|_1 \leq
  C\}.
\]
where $C$ will be a parameter of our algorithm. Since $\Lambda$ is a compact
and convex set, the minmax condition continues to hold \citep{sion1958}:
\begin{equation}
  \min_{D\in \Delta_{\cH(S)}}\, \max_{\lambda\in \Lambda} \cL(D,
  \lambda) = \max_{\lambda\in \Lambda}\, \min_{D\in \Delta_{\cH(S)}}
  \cL(D, \lambda) \label{minmax}
\end{equation}

If we knew an upper bound $C$ on the $\ell_1$ norm of the optimal dual
solution, then this restriction on the dual solution would not change
the minmax solution of the program. We do not in general know such a
bound. However, we can show that even though we restrict the dual
variables to lie in a bounded set, any approximate minmax solution to
\Cref{minmax} is also an approximately optimal and approximately
feasible solution to the original fair ERM problem.

\begin{restatable}{theorem}{approxmm}\label{thm:approxmm}
  Let $(\hat D, \hat \lambda)$ be a $\nu$-approximate minmax solution
  to the $\Lambda$-bounded Lagrangian problem in the sense that
\[
  \cL(\hat D, \hat \lambda) \leq \min_{D\in \Delta_{\cH(S)}}\cL(D,
  \hat \lambda) + \nu \quad \mbox{and,} \quad \cL(\hat D, \hat
  \lambda) \geq \max_{\lambda\in \Lambda} \cL(\hat D, \lambda) - \nu.
\]
Then $err(\hat D, \cP) \leq \OPT + 2\nu$ and for any $g\in \cG(S)$,
\[
  \alpha_{FP}(g, \cP)\; \beta_{FP}(g, \hat D, \cP) \leq \gamma + \frac{1 + 2\nu}{C}.
\]
\end{restatable}

\subsection{Zero-Sum Game Formulation}
To compute an approximate minmax solution, we will first view
\Cref{minmax} as the following two player zero-sum matrix game. The
Learner (or the minimization player) has pure strategies corresponding
to $\cH$, and the Auditor (or the maximization player) has pure
strategies corresponding to the set of vertices
$\Lambda_{\text{pure}}$ in $\Lambda$ --- more precisely, each vertex
or pure strategy either is the all zero vector or consists of a choice
of a $g \in \cG(S)$, along with the sign $+$ or $-$ that the
corresponding $g$-fairness constraint will have in the
Lagrangian. More formally, we write
\[
  \Lambda_{\text{pure}} = \{\lambda \in \Lambda \mbox{ with }
  \lambda_g^\bullet = C \mid g\in \cG(S), \bullet \in \{\pm\}\} \cup
  \{\mathbf{0}\}
\]
Even though the number of pure strategies scales linearly with
$|\cG(S)|$, our algorithm will never need to actually represent such
vectors explicitly.  Note that any vector in $\Lambda$ can be written
as a convex combination of the maximization player's pure strategies, or
in other words: as a mixed strategy for the Auditor. For any pair of
actions $(h, \lambda)\in \cH\times \Lambda_{\text{pure}}$, the payoff
is defined as
\[
  U(h, \lambda) = err(h, \cP) + \sum_{g\in \cG(S)}
  \left(\lambda_g^+\Phi_+(h, g) + \lambda_g^-\Phi_-(h, g) \right).
\]

\begin{claim}
  Let $D\in \Delta_{\cH(S)}$ and $\lambda\in \Lambda$ such that
  $(p, \lambda)$ is a $\nu$-approximate minmax equilibrium in the
  zero-sum game defined above. Then $(p, \lambda)$ is also a
  $\nu$-approximate minmax solution for \Cref{minmax}.
\end{claim}

Our problem reduces to finding an approximate equilibrium for this
game. A key step in our solution is the ability to compute best
responses for both players in the game, which we now show can be
solved by the cost-sensitive classication (CSC) oracles.

\paragraph{Learner's best response as CSC.} Fix any mixed strategy (dual
solution) $ \lambda\in \Lambda$ of the Auditor. The Learner's
best response is given by:
\begin{equation} \label{eq:br} \argmin_{D \in \Delta_{\cH(S)}} \;
  err(h, \cP) + \sum_{g\in \cG(S)} \left( \lambda_g^+ \Phi_+(D, g) +
    \lambda_g^- \Phi_-(D, g) \right)
\end{equation}
Note that it suffices for the Learner to optimize over deterministic
classifiers $h \in \cH$, rather than distributions over
classifiers. This is because the Learner is solving a linear
optimization problem over the simplex, and so always has an optimal
solution at a vertex (i.e. a single classifier $h \in \cH$). We can
reduce this problem to one that can be solved with a single call to a
{CSC oracle}.  In
particular, we can assign costs to each example $(X_i, y_i)$ as
follows:
\begin{itemize}
\item if $y_i = 1$, then $c_i^0 = 0$ and $c_i^1 = - \frac{1}{n}$;
\item otherwise, $c_i^0 = 0$ and
  \begin{align}
    c_i^1 = \frac{1}{n} &+ \frac{1}{n}\sum_{g\in \cG(S)}
                          (\lambda_g^+ - \lambda_g^-) \left(\Pr[g(x) = 1\mid y = 0] -  \mathbf{1}[g(x_i) = 1]\right)
  \end{align}
\end{itemize}

\noindent Given a fixed set of dual variables $\lambda$, we will write
$\LC(\lambda) \in \RR^n$ to denote the vector of costs for labelling
each datapoint as $1$. That is, $\LC(\lambda)$ is the vector such that
for any $i \in [n]$, $\LC(\lambda)_i = c_i^1$.

\begin{remark}
  Note that in defining the costs above, we have translated them from
  their most natural values so that the cost of labeling any example
  with 0 is 0. In doing so, we recall that by Claim \ref{claim:csc},
  the solution to a cost-sensitive classification problem is invariant
  to translation. As we will see, this will allow us to formulate the
  learner's optimization problem as a low-dimensional linear
  optimization problem, which will be important for an efficient
  implementation of follow the perturbed leader. In particular, if we
  find a hypothesis that produces the $n$ labels
  $y = (y_1,\ldots,y_n)$ for the $n$ points in our dataset, then the
  cost of this labelling in the CSC problem is by construction
  $\langle \LC(\lambda), y \rangle$. 
\end{remark}

\paragraph{Auditor's best response as CSC.}

Fix any mixed strategy (primal solution) $p \in \Delta_{\cH(S)}$ of
the Learner. The Auditor's best response is given by:
\begin{equation} \label{eq:abr} \argmax_{\lambda \in \Lambda} \;
  err(D, \cP) + \sum_{g\in \cG(S)} \left( \lambda_g^+ \Phi_+(D, g) +
    \lambda_g^- \Phi_-(D, g) \right) = \argmax_{\lambda \in \Lambda}
  \sum_{g\in \cG(S)} \left( \lambda_g^+ \Phi_+(D, g) + \lambda_g^-
    \Phi_-(D, g) \right)
\end{equation}

To find the best response, consider the problem of computing
$(\hat g, \hat \bullet) = \argmax_{(g, \bullet)}\Phi_\bullet(D, g)$.
There are two cases. In the first case, $p$ is a strictly feasible primal
solution: that is $\Phi_{\hat \bullet} (D, \hat g) < 0$. In this case, the
solution to \eqref{eq:abr} sets $\lambda =
\mathbf{0}$. Otherwise, if $p$ is not strictly feasible, then by the following
\Cref{lem:dualbr} the best response is to set
$\lambda_{\hat g}^{\hat \bullet} = C$ (and all other coordinates to 0).

\begin{restatable}{lemma}{dualbr}\label{lem:dualbr}
  Fix any $\overline D\in \Delta_{\cH(S)}$ such that that
  $\max_{g\in \cG(S)} \{\Phi_+(\overline D, g), \Phi_-(\overline D,
  g)\} > 0$. Let $\lambda' \in \Lambda$ be vector with one non-zero
  coordinate $(\lambda')_{g'}^{\bullet'} = C$, where
  \[
    (g',\bullet') = \argmax_{(g, \bullet)\in \cG(S)\times \{\pm\}}
    \{\Phi_\bullet(\overline D, g) \}
  \]
  Then
  $\cL(\overline D, \lambda') \geq \max_{\lambda\in
    \Lambda}\cL(\overline D, \lambda)$.
\end{restatable}

Therefore, it suffices to solve for
$\argmax_{(g, \bullet)}\Phi_\bullet(D, g)$. We proceed by solving
$\argmax_{g} \Phi_+(D, g)$ and $\argmax_{g} \Phi_-(D, g)$ separately:
 both problems can be reduced to a cost-sensitive classification
problem. To solve for $\argmax_{g} \Phi_+(D, g)$ with a CSC oracle, we
assign costs to each example $(X_i, y_i)$ as follows:
\begin{itemize}
\item if $y_i = 1$, then $c_i^0 = 0$ and $c_i^1 = 0$;
\item otherwise, $c_i^0 = 0$ and
  \begin{align}
    c_i^1 = \frac{-1}{n}\left[ \Ex{h\sim D}{\FP(h)} - \Ex{h\sim D}{h(X_i)}  \right]
  \end{align}
\end{itemize}

To solve for $\argmax_{g} \Phi_-(D, g)$ with a CSC oracle, we assign
the same costs to each example $(X_i, y_i)$, except when $y_i = 0$,
labeling ``1'' incurs a cost of
\[
  c_i^1 = \frac{-1}{n}\left[ \Ex{h\sim D}{h(X_i)} - \Ex{h\sim
      D}{\FP(h)} \right]
\]

\subsection{Solving the Game with No-Regret Dynamics}

To compute an approximate equilibrium of the zero-sum game, we will
simulate the following \emph{no-regret dynamics} between the Learner
and the Auditor over rounds: over each of the $T$ rounds, the Learner
plays a distribution over the hypothesis class according to a
\emph{no-regret} learning algorithm (Follow the Perturbed Leader), and the Auditor plays an
approximate best response against the Learner's distribution for that
round. By the result of \cite{FS96}, the average plays of both players
over time converge to an approximate equilibrium of the game, as long
as the Learner has low regret.

\begin{theorem}[\cite{FS96}]\label{fs}
  Let $D^1, D^2 , \ldots , D^T\in \Delta_{\cH(S)}$ be a sequence of
  distributions played by the Learner, and let
  $\lambda^1, \lambda^2 , \ldots , \lambda^T \in
  \Lambda_{\textrm{pure}}$ be the Auditor's sequence of approximate
  best responses against these distributions respectively. Let
  $\overline D = \frac{1}{T} \sum_{t=1}^T D^t$ and
  $\overline \lambda = \frac{1}{T}\sum_{t=1}^T \lambda^t$ be the two
  players' empirical distributions over their strategies. Suppose that
  the regret of the Learner satisfies
  \[
    \sum_{t= 1}^T \Ex{h\sim D^t}{U(h, \lambda^t)} - \min_{h\in
      \cH(S)}\sum_{t=1}^T U(h, \lambda^t) \leq \gamma_L T \quad \mbox{
      and} \quad \max_{\lambda\in \Lambda} \sum_{t= 1}^T \Ex{h\sim
      D^t}{U(h, \lambda)} - \sum_{t=1}^T \Ex{h\sim D^t}{U(h,
      \lambda^t)} \leq \gamma_A T.
  \]
  Then $(\overline D, \overline \lambda)$ is an
  $(\gamma_L + \gamma_A)$-approximate minimax equilibrium of the game.
\end{theorem}

Our Learner will play using the Follow the Perturbed Leader (FTPL), which gives a no-regret guarantee. In order to implement FPTL, we will first need to formulate the Learner's best
response problem as a linear optimization problem over a low dimensional space. For each round $t$,
let $\overline\lambda^t = \sum_{s < t} \lambda^s$ be the vector representing the sum of the actions played by the auditor over previous rounds, and recall that
$\LC(\overline \lambda^t)$ is the cost vector given by our
cost-sensitive classification reduction. Then
the Learner's best response problem against $\overline \lambda^t$ is
the following linear optimization problem
\[
  \min_{h \in \cH(S)} \langle \LC(\overline\lambda^t), h \rangle.
\]
To run the FTPL algorithm, the Learner will optimize a ``perturbed''
version of the problem above. In particular, the Learner will play a
distribution $D^t$ over the hypothesis class $\cH(S)$ that is implicitely defined by the following sampling operation. To sample a hypothesis $h$ from $D^t$, the learner solves the following randomized optimization
problem:
\begin{equation}
  \min_{h \in \cH(S)} \langle \LC(\overline\lambda^t), h \rangle + \frac{1}{\eta}\langle
  \xi, h \rangle, \label{pbr-eq}
\end{equation}
where $\eta$ is a parameter and $\xi$ is a noise vector drawn from the
uniform distribution over $[0, 1]^n$. Note that while it is
intractable to explicitly represent the distribution $D^t$ (which has support size
scaling with $|\cH(S)|$), we can sample from $D^t$ efficiently given
access to a cost-sensitive classification oracle for $\cH$. By
instantiating the standard regret bound of FTPL for online linear
optimization (\Cref{thm:ftpl-regret}), we get the following regret
bound for the Learner.

\begin{restatable}{lemma}{learneregret}\label{learner-regret}
  Let $T$ be the time horizon for the no-regret dynamics. Let
  $D^1, \ldots, D^T$ be the sequence of distributions maintained by
  the Learner's FTPL algorithm with
  $\eta = \frac{n}{(1+C)} \sqrt{\frac{1}{\sqrt{n} T}}$, and
  $\lambda^1, \ldots, \lambda^T$ be the sequence of plays by the
  Auditor. Then
 \[
   \sum_{t= 1}^T \Ex{h\sim D^t}{U(h, \lambda^t)} - \min_{h\in
     \cH(S)}\sum_{t=1}^T U(h, \lambda^t) \leq 2 n^{1/4}
   {(1+C)} \sqrt{T}
  \]
\end{restatable}

Now we consider how the Auditor (approximately) best responds to the
distribution $D^t$. The main obstacle is that we do not have an
explicit representation for $D^t$. Thus, our first step is to
approximate $D^t$ with an explicitly represented sparse distribution
$\hat D^t$. We do that by drawing $m$ i.i.d. samples from $D^t$, and
taking the empirical distribution $\hat D^t$ over the sample. The
Auditor will best respond to this empirical distribution $\hat
D^t$. To show that any best response to $\hat D^t$ is also an
approximate best response to $D^t$, we will rely on the following
uniform convergence lemma, which bounds the difference in expected
payoff for any strategy of the auditor, when played against $D^t$ as
compared to $\hat D^t$.

\begin{restatable}{lemma}{samplingerr}\label{lem:sampling-err}
  Fix any $\xi, \delta \in (0, 1)$ and any distribution $D$ over
  $\cH(S)$.  Let $h^1, \ldots, h^m$ be $m$ i.i.d. draws from $p$, and
  $\hat D$ be the empirical distribution over the realized
  sample. Then with probability at least $1 - \delta$ over the random
  draws of $h^j$'s, the following holds,
  \[
    \max_{\lambda\in \Lambda} \left| \Ex{h\sim \hat D}{U(h, \lambda)}
      - \Ex{h\sim D}{U(h, \lambda)} \right| \leq \xi,
  \]
  as long as $m \geq c_0\frac{C^2\left(\ln(1/\delta) + d_2\ln(n)\right)}{\xi^2}$
  for some absolute constant $c_0$ and $d_2 = \VCD(\cG)$.
\end{restatable}

\noindent Using \Cref{lem:sampling-err}, we can derive a regret bound
for the Auditor in the no-regret dynamics.

\begin{restatable}{lemma}{auditorbr}\label{lem:auditor-br}
  Let $T$ be the time horizon for the no-regret dynamics. Let
  $D^1, \ldots, D^T$ be the sequence of distributions maintained by
  the Learner's FTPL algorithm. For each $D^t$, let $\hat D^t$ be the
  empirical distribution over $m$ i.i.d. draws from $D^t$. Let
  $\lambda^1, \ldots, \lambda^T$ be the Auditor's best responses
  against $\hat D^1, \ldots , \hat D^T$. Then with probability
  $1 - \delta$,
 \[
   \max_{\lambda\in \Lambda}\sum_{t= 1}^T \Ex{h\sim D^t}{U(h,
     \lambda)} - \sum_{t=1}^T\Ex{h \sim D^t}{ U(h, \lambda^t) }\leq T
   \sqrt{\frac{c_0 C^2 (\ln(T/\delta) + d_2 \ln(n))}{m}}
  \]
  for some absolute constant $c_0$ and $d_2 = \VCD(\cG)$.
\end{restatable}

Finally, let $\overline D$ and $\overline \lambda$ be the average of
the strategies played by the two players over the course of the
dynamics. Note that $\overline D$ is an average of many
\emph{distributions} with large support, and so $\overline D$ itself
has support size that is too large to represent explicitely. Thus, we
will again approximate $\overline D$ with a sparse distribution
$\hat D$ estimated from a sample drawn from $\overline D$. Note that
we can efficiently \emph{sample} from $\overline D$ given access to a
CSC oracle. To sample, we first uniformly randomly select a round
$t\in [T]$, and then use the CSC oracle to solve the sampling problem
defined in \eqref{pbr-eq}, with the noise random variable $\xi$
freshly sampled from its distribution. The full algorithm is described
in \Cref{alg:fairnr} and we present the proof for \Cref{thm:polytime}
below. 

\begin{algorithm}[h]
  \caption{\bf{FairNR: Fair No-Regret Dynamics}}
 \label{alg:fairnr}
  \begin{algorithmic}
   \STATE{\textbf{Input:} distribution $\cP$ over $n$ labelled data
      points, CSC oracles $\CSC(\cH)$ and $\CSC(\cG)$, dual bound $C$,
      and target accuracy parameter $\nu, \delta$ }

    \STATE{\textbf{Initialize}: Let
      $C = 1/\nu$,
      $\overline \lambda^0 = \mathbf{0}$,
      $\eta = \frac{n}{(1+C)} \sqrt{\frac{1}{\sqrt{n} T}}$,
      $$
      m = \frac{\left(\ln(2T/\delta)d_2\ln(n)\right) C^2 c_0
        T}{\sqrt{n} (1+C)^2 \ln(2/\delta)} \quad \mbox{and, }\quad T =
      \frac{4\sqrt{n} \ln(2/\delta)}{\nu^4}
      $$ }
    \STATE{\textbf{For} $t = 1, \ldots, T$:}

    \INDSTATE{Sample from the Learner's FTPL distribution:}
    \INDSTATE[2]{\textbf{For} $s = 1 , \ldots m$:} \INDSTATE[3]{Draw a
      random vector $\xi^s$ uniformly at random from $[0, 1]^n$}
    \INDSTATE[3]{Use the oracle $\CSC(\cH)$ to compute
      $h^{(s, t)} = \argmin_{h \in \cH(S)} \langle \LC(\overline\lambda^{(t-1)}), h
      \rangle + \frac{1}{\eta}\langle \xi^s, h \rangle$ }
    \INDSTATE[2]{Let $\hat D^t$ be the empirical distribution over
      $\{h^{s, t}\}$}
\STATE{}

\INDSTATE{Auditor best responds to $\hat D^t$:}

\INDSTATE[2]{Use the oracle $\CSC(\cG)$ to compute
  $\lambda^t = \argmax_{\lambda} \Ex{h\sim \hat D}{U(h, \lambda)}$

}

\INDSTATE{{Update}: {Let}
  $\overline \lambda^t = {\sum_{t'\leq t} \lambda^{t'}}$ }

    \STATE{Sample from the average distribution
      $\overline D = \sum_{t=1}^T D^t$: }
    \INDSTATE{\textbf{For} $s = 1 , \ldots m$:}
\INDSTATE[2]{Draw a random number $r\in [T]$ and a random vector $\xi^s$ uniformly at random from $[0, 1]^n$}
    \INDSTATE[2]{Use the oracle $\CSC(\cH)$ to compute
      $h^{(r, t)} = \argmin_{h \in \cH(S)} \langle \LC(\overline\lambda^{(r-1)}), h
      \rangle + \frac{1}{\eta}\langle \xi^s, h \rangle$ }
    \INDSTATE[1]{Let $\hat D$ be the empirical distribution over
      $\{h^{r, t}\}$}
    \STATE{\textbf{Output}: $\hat D$ as a randomized classifier}
    \end{algorithmic}
  \end{algorithm}

\begin{proof}[Proof of \Cref{thm:polytime}]
  By \Cref{thm:approxmm}, it suffices to show that with probability at
  least $1 - \delta$, $(\hat D, \overline \lambda)$ is a
  $\nu$-approximate equilibrium in the zero-sum game. As a first step,
  we will rely on \Cref{fs} to show that
  $(\overline D, \overline \lambda)$ forms an approximate equilibrium.

  By \Cref{learner-regret}, the regret of the sequence
  $D^1, \ldots , D^T$ is bounded by:
  \[
    \gamma_L = \frac{1}{T}\left[\sum_{t= 1}^T \Ex{h\sim D^t}{U(h,
        \lambda^t)} - \min_{h\in \cH(S)}\sum_{t=1}^T U(h, \lambda^t)
    \right] \leq \frac{2 n^{1/4} {(1+C)}}{\sqrt{T}}
  \]


  By \Cref{lem:auditor-br}, with probability $1 - \delta/2$, we have
\begin{align*}
  \gamma_A &\leq \sqrt{\frac{c_0 C^2 (\ln(2T/\delta) + d_2
        \ln(n))}{m}}
\end{align*}
  We will condition on this upper-bound event on $\gamma_A$ for the
  rest of this proof, which is the case except with probability
  $\delta/2$. By \Cref{fs}, we know that the average plays
  $(\overline D, \overline \lambda)$ form an
  $(\gamma_L+\gamma_A)$-approximate equilibrium.

  Finally, we need to bound the additional error for outputting the
  sparse approximation $\hat D$ instead of $\overline D$. We can
  directly apply \Cref{lem:sampling-err}, which implies that except
  with probability $\delta/2$, the pair $(\hat D, \overline \lambda)$
  form a $R$-approximate equilibrium, with
  \begin{align*}
    R &\leq \gamma_A + \gamma_L + \frac{\sqrt{c_0 C^2(\ln(2/\delta)
        +d_2\ln(n))}}{\sqrt{m}}
  \end{align*}
  Note that $R\leq \nu$ as long as we have $C = 1/\nu$,
  \[
    m = \frac{\left(\ln(2T/\delta)d_2\ln(n)\right) C^2 c_0 T}{\sqrt{n} (1+C)^2 \ln(2/\delta)}
 \quad
    \mbox{and, }\quad T = \frac{4\sqrt{n} \ln(2/\delta)}{\nu^4}
  \]
 This completes our proof.
\end{proof}

\section{Experimental Evaluation}
\label{sec:empirical}

We now describe an experimental evaluation of our proposed algorithmic framework
on a dataset in which fairness is a concern, due to the preponderance of racial and
other sensitive features. For far more detailed experiments on four
real datasets investigating the convergence properties
of our algorithm, evaluating its accuracy vs. fairness tradeoffs, and comparing our approach
to the recent algorithm of \cite{reductions}, we direct the reader to \cite{experimental}. Python code and an illustrative Jupyter notebook
are provided \href{https://github.com/algowatchpenn/GerryFair}{here} (https://github.com/algowatchpenn/GerryFair).

While the no-regret-based algorithm described in the last
section enjoys provably polynomial time convergence, for the experiments we
instead implemented a simpler yet effective algorithm based on {\em Fictitious Play\/} dynamics.
We first describe and discuss this modified algorithm.

\subsection{Solving the Game with Fictitious Play}

Like the algorithm given in the last section, the algorithm we implemented
works by simulating a game dynamic that converges to Nash equilibrium
in the zero-sum game that we derived, corresponding to the Fair ERM problem.
Rather than using a no-regret dynamic, we instead use  a
simple iterative procedure known as \emph{Fictitious Play}
\citep{brown_1949}. Fictitious Play dynamics has the benefit of being more practical to
implement: at each round, both players simply need to compute a single best response to the empirical play of their opponents, and this optimization requires only a single call to a CSC oracle.
In contrast, the FTPL dynamic we gave in the previous section requires making many calls to a
CSC oracle per round --- a computationally expensive process ---
in order to find a sparse approximation to the Learner's mixed strategy at that round.
Fictitious Play also has the benefit of being
deterministic, unlike the randomized sampling required in the FTPL no-regret dynamic,
thus eliminating a source of experimental variance.

The disadvantage is that Fictitious Play is only known to converge to
equilibrium in the limit~\cite{R51}, rather than in a polynomial number of rounds (though it is
conjectured to converge quickly under rather general circumstances; see \cite{DP14} for a recent discussion).
Nevertheless, this is the algorithm that we use in our
experiments --- and as we will show, it performs well on real data,
despite the fact that it has weaker theoretical guarantees compared to the algorithm we presented in the last section.

Fictitious play proceeds in rounds, and in every round each
player chooses a best response to his opponent's empirical history of play across previous rounds, by treating it as the mixed strategy that randomizes uniformly over the empirical history.
Pseudocode for the implemented algorithm is given below.

\begin{algorithm}[h]
  \caption{\bf{FairFictPlay: Fair Fictitious Play}}
 \label{alg:fairfict}
  \begin{algorithmic}
    \STATE{\textbf{Input:} distribution $\cP$ over the labelled data
      points, CSC oracles $\CSC(\cH)$ and $\CSC(\cG)$ for the classes
      $\cH(S)$ and $\cG(S)$ respectively, dual bound $C$, and number
      of rounds $T$}

    \STATE{\textbf{Initialize}: set $h^0$ to be some classifier in
      $\cH$, set $\lambda^0$ to be the zero vector. Let
      $\overline D$ and $\overline \lambda$ be the point distributions that put
      all their mass on $h^0$ and $\lambda^0$ respectively.}

    \STATE{\textbf{For} $t = 1, \ldots, T$:}

    \INDSTATE{\textbf{Compute the empirical play distributions}:}

    \INDSTATE[2]{\textbf{Let} $\overline D$ be the uniform distribution
      over the set of classifiers $\{h^0, \ldots, h^{t-1}\}$ }

    \INDSTATE[2]{\textbf{Let}
      $\overline \lambda = \frac{\sum_{t'< t} \lambda^{t'}}{t}$ be the
      auditor's empirical dual vector}

    \INDSTATE{\textbf{Learner best responds}:
      Use the oracle $\CSC(\cH)$ to compute
      $h^{t} = \argmin_{h \in \cH(S)} \langle \LC(\overline\lambda), h
      \rangle$
}

\INDSTATE{\textbf{Auditor best responds}: Use the oracle $\CSC(\cG)$
  to compute
  $\lambda^t = \argmax_{\lambda} \Ex{h\sim \overline D}{U(h,
    \lambda)}$
}

    \STATE{\textbf{Output:} the final empirical distribution
      $\overline D$ over classifiers}

    \end{algorithmic}
  \end{algorithm}

\subsection{Description of Data}


The dataset we use for our experimental valuation is known as the ``Communities and Crime'' (C\&C) dataset,
available at the UC Irvine Data Repository\footnote{\url{http://archive.ics.uci.edu/ml/datasets/Communities+and+Crime}}.
Each record in this dataset describes the aggregate demographic properties of a different U.S. community;
the data combines socio-economic data from the 1990 US Census, law enforcement data from the 1990 US LEMAS survey,
and crime data from the 1995 FBI UCR. The total number of records is 1994, and the number of features is 122.
The variable to be predicted is the rate of violent crime in the community.

While there are larger and more recent datasets in which subgroup fairness is a potential concern, there
are properties of the C\&C dataset that make it particularly appealing for the initial experimental evaluation
of our proposed algorithm. Foremost among these is the relatively high number of sensitive or protected attributes,
and the fact that they are real-valued (since they represent aggregates in a community rather than specific individuals). This means there is a very large number of protected sub-groups that can be defined over them.
There are distinct continuous features measuring the percentage or per-capita representation of multiple
racial groups (including white, black, Hispanic, and Asian) in the community, each of which can
vary independently of the others. Similarly, there are continuous features measuring the average per capita incomes
of different racial groups in the community, as well as features measuring the percentage of each community's
police force that falls in each of the racial groups. Thus restricting to features capturing race statistics and a couple of
related ones (such as the percentage of residents who do not speak English well),
we obtain an 18-dimensional space of real-valued protected attributes. We note that the C\&C dataset has numerous other
features that arguably could or should be protected as well (such as gender features),
which would raise the dimensionality of the protected
subgroups even
further. \footnote{Ongoing experiments on other datasets where fairness is
a concern will be reported on in a forthcoming experimental paper.}

We convert the real-valued rate of violent crime in each community to a binary label
indicating whether the community is in the 70th percentile of that value, indicating that it is
a relatively high-crime community. Thus the strawman baseline that always predicts 0 (lower crime)
has error approximately 30\% or 0.3 on this classification problem. We chose the 70th percentile
since it seems most natural to predict the highest crime rates.

As in the theoretical sections of the paper, our main interest and emphasis is on the effectiveness
of our proposed algorithm \FFP on a given dataset, including:
\begin{itemize}
	\item Whether the algorithm in fact converges, and does so in a feasible amount of computation.
		Recall that formal convergence is only guaranteed under the assumption of oracles that
		do not exist in practice, and even then is only guaranteed asymptotically.
	\item Whether the classifier learned by the algorithm has nontrivial accuracy, as well as
		strong subgroup fairness properties.
	\item Whether the algorithm and dataset permits nontrivial tuning of the trade-off between accuracy and
		subgroup fairness.
\end{itemize}
As discussed in Section~\ref{subsec:gen},
we note that all of these issues can be investigated entirely in-sample, without concern for generalization
performance. Thus for simplicity, despite the fact that our algorithm enjoys all the usual generalization
properties depending on the VC dimension of the Learner's hypothesis space and the Auditor's
subgroup space (see Theorems~\ref{thm:FP-uc} and~\ref{thm:SP-uc}), we report all results here on the full C\&C dataset of 1994 points, treating
it as the true distribution of interest.

\subsection{Algorithm Implementation}

The main details in the implementation of \FFP are the identification of the
model classes for Learner and Auditor, the implementation of the cost sensitive classification oracle and auditing oracle, and the identification of the protected features for Auditor.
For our experiments, at each round Learner chooses a linear threshold function over all 122 features. We implement the cost sensitive classification oracle via a two stage regression procedure.  In particular, the inputs to the cost sensitive classification oracle are cost vectors $c_0, c_1$, where the $i^{th}$ element of $c_k$
is the cost of predicting $k$ on datapoint $i$. We train two linear regression
models $r_0, r_1$ to predict $c_0$ and $c_1$ respectively, using all $122$ features. Given a new point $x$, we predict the cost
of classifying $x$ as $0$ and $1$ using our regression models: these predictions are $r_0(x)$ and $r_1(x)$ respectively. Finally we output the prediction $\hat{y}$ corresponding to lower predicted cost:
$\hat{y} = \text{argmin}_{i \in \{0,1\}}r_i(x)$.

Auditor's model class consists of all linear threshold functions over just the 18 aforementioned protected
race-based attributes. As per the algorithm, at each iteration $t$ Auditor attempts to find a subgroup on which the false positive rate is substantially different than the base rate, given the Learner's randomized classifier so far. We implement the auditing oracle  by treating it as a weighted regression problem in which
the goal is find a linear function (which will be taken to define the subgroup) that on the negative examples, can predict the Learner's probabilistic classification on each point. We use the same regression subroutine as Learner does, except that Auditor only has access to the $18$ sensitive features, rather than all $122$.

Recall that in addition to the choices of protected attributes and model classes for Learner and
Auditor, \FFP has a parameter $C$, which is a bound on the norm of the dual variables for Auditor (the dual player).
While the theory does not provide an explicit bound or guide for choosing $C$, it needs to be large enough
to permit the dual player to force the minmax value of the game.
For our experiments we chose $C = 10$, which despite being a relatively small value seems to
suffice for (approximate) convergence.

The other and more meaningful parameter of the algorithm is the bound $\gamma$
in the Fair ERM optimization problem implemented by the game, which controls the
amount of unfairness permitted. If on a given round the subgroup disparity found by the
Auditor is greater than $\gamma$, the Learner must react by adding a fairness penalty for this
subgroup to its objective function; if it is smaller than $\gamma$, the Learner can ignore it
and continue to optimize its previous objective function. Ideally, and as we shall see, varying
$\gamma$ allows us to trace out a menu of trade-offs between accuracy and fairness.

\subsection{Results}

Particularly in light of the gaps between the idealized theory and the actual implementation, the most basic
questions about \FFP are whether it converges at all, and if so, whether it converges to ``interesting''
models --- that is, models with both nontrivial classification error (much better than the 30\% or 0.3 baserate), and nontrivial
subgroup fairness (much better than ignoring fairness altogether). We shall see that at least for the C\&C dataset,
the answers to these questions is strongly affirmative.

\begin{figure*}[h]
\centering
\subfigure[]{\includegraphics[scale=0.35]{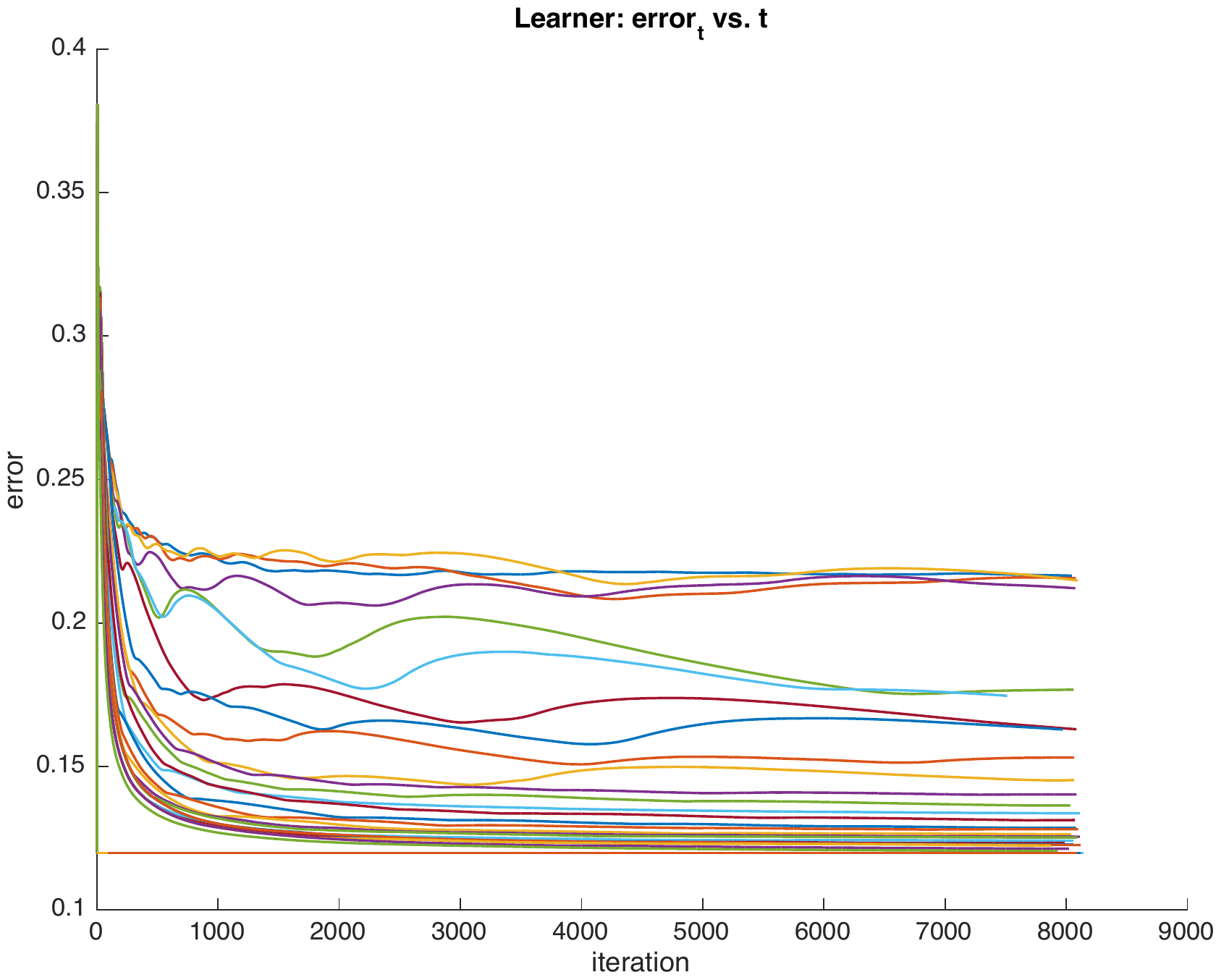}}
\subfigure[]{\includegraphics[scale=0.35]{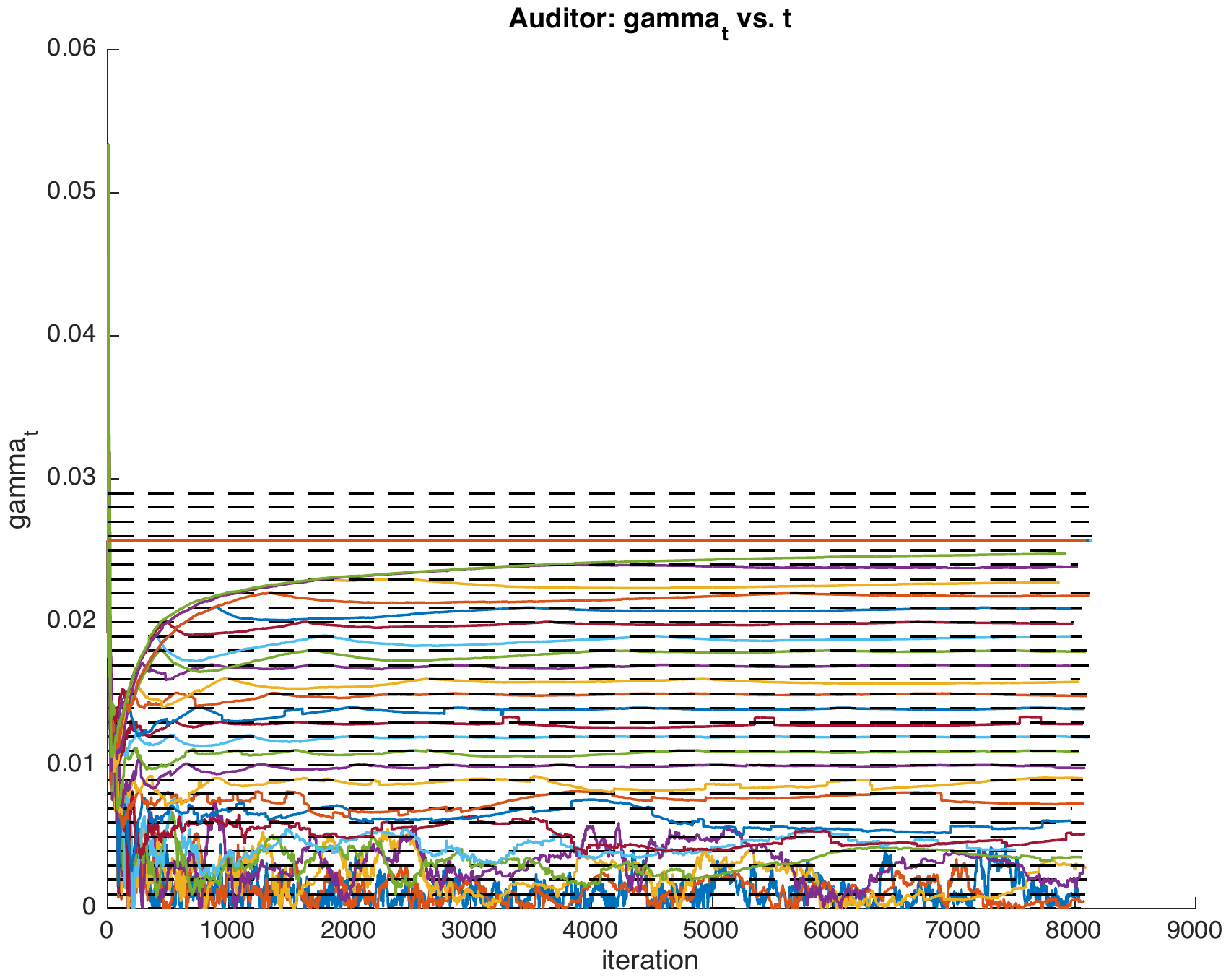}}
\caption{
Evolution of the error and unfairness of Learner's classifier across iterations,
for varying choices of $\gamma$.
(a) Error $\epsilon_t$ of Learner's model vs iteration $t$.
(b) Unfairness $\gamma_t$ of subgroup found by Auditor vs. iteration $t$, as measured
	by Definition~\ref{fp-fair}.
See text for details.
}
\label{fig:error}
\end{figure*}

We begin by examining the evolution of the error and unfairness of Learner's model.
In the left panel of Figure~\ref{fig:error} we show the error of the model found
by Learner vs. iteration for values of $\gamma$ ranging from 0 to 0.029.
Several comments are in order.

First, after an initial period in which there is a fair amount of oscillatory
behavior, by 6000 iterations most of the curves have largely flattened out,
and by 8,000 iterations it appears most but not all have reached approximate convergence.
Second, while the top-to-bottom ordering of these error curves is approximately
aligned with decreasing $\gamma$ --- so larger $\gamma$ generally results in lower error,
as expected --- there are many violations of this for small $t$, and even a few at large
$t$. Third, and as we will examine more closely shortly, the converged values
at large $t$ do indeed exhibit a range of errors.

In the right panel of Figure~\ref{fig:error}, we show the corresponding unfairness $\gamma_t$ of the subgroup
found by the Auditor at each iteration $t$ for the same runs and values of the parameter $\gamma$ (indicated by
horizontal dashed lines), with the same
color-coding as for the left panel. Now the ordering is generally reversed --- larger values of $\gamma$
generally lead to higher $\gamma_t$ curves, since the fairness constraint on the Learner is weaker.
We again see a great deal of early oscillatory behavior, with most $\gamma_t$ curves then eventually settling
at or near their corresponding input $\gamma$ value, as Learner and Auditor engage in a back-and-forth struggle
for lower error for Learner and $\gamma$-subgroup fairness for Auditor.

\begin{center}
\begin{figure*}[h]
\centering
\subfigure[]{\includegraphics[scale=0.35]{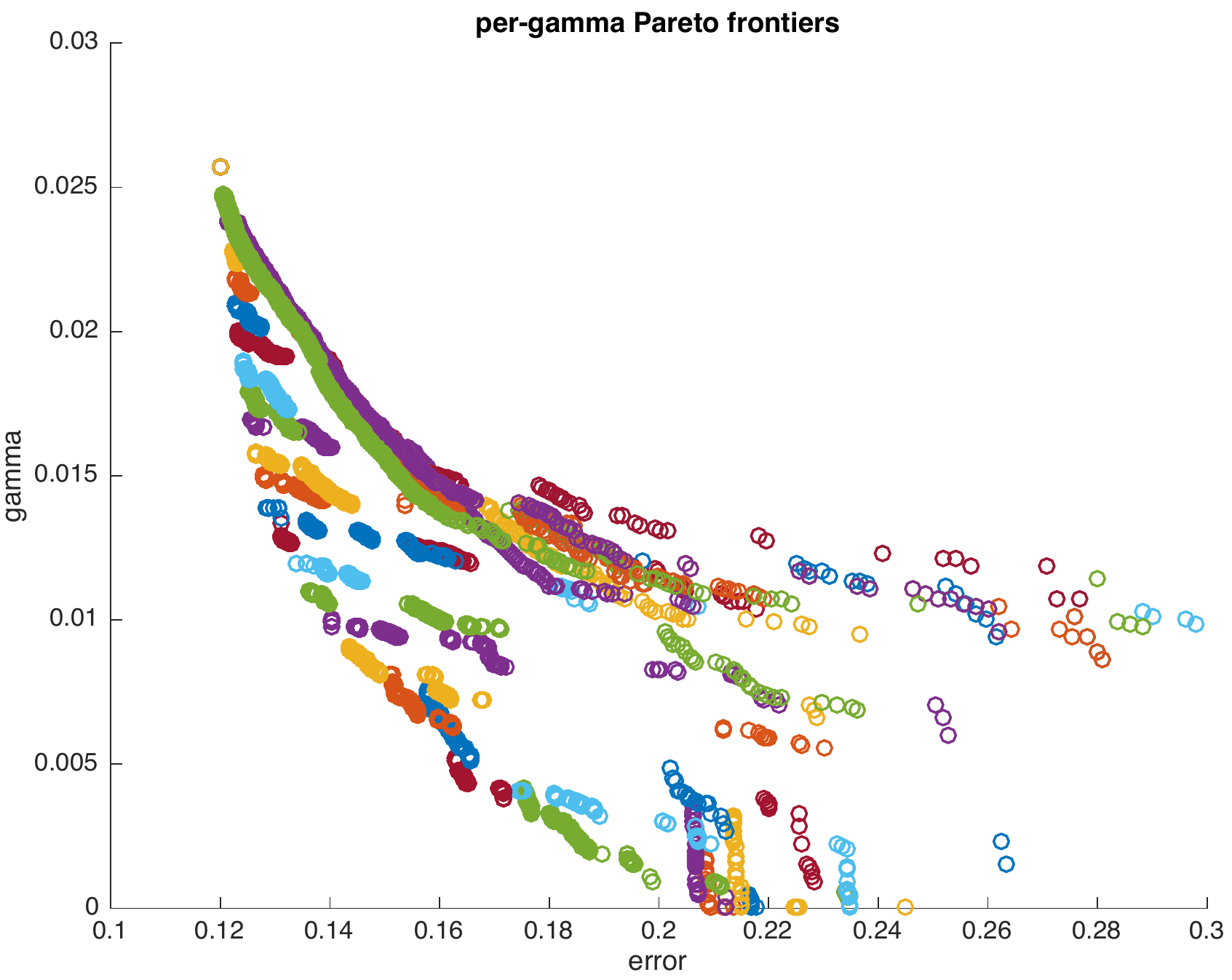}}
\subfigure[]{\includegraphics[scale=0.35]{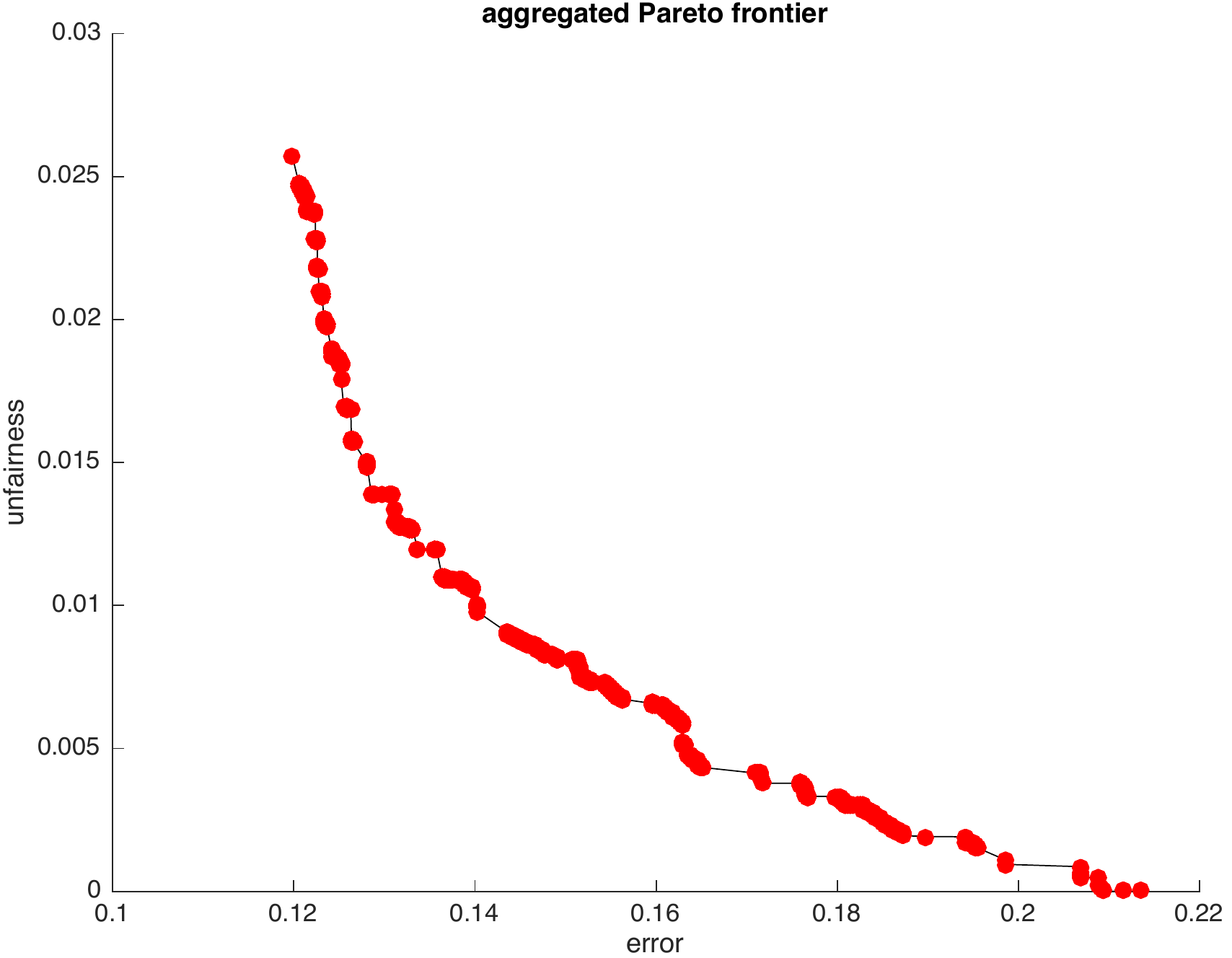}}
\caption{
(a) Pareto-optimal error-unfairness values, color coded by varying values of
	the input parameter $\gamma$.
(b) Aggregate Pareto frontier across all values of $\gamma$. Here the $\gamma$ values
cover the same range but are sampled more densely to get a smoother frontier.
See text for details.
}
\label{fig:fairness}
\end{figure*}
\end{center}

For any choice of the parameter $\gamma$, and each iteration $t$, the two panels of Figure~\ref{fig:error}
yield a pair of realized values $\langle \epsilon_t, \gamma_t \rangle$ from the experiment, corresponding to
a Learner model whose error is $\epsilon_t$, and for which the worst subgroup the Auditor was able to find
had unfairness $\gamma_t$.
The set of all
$\langle \epsilon_t, \gamma_t \rangle$
pairs across all runs or $\gamma$ values thus represents the
different trade-offs between error and unfairness found by our algorithm on the data.
Most of these pairs are of course Pareto-dominated by other pairs, so we are primarily
interested in the undominated frontier.

In the left panel of Figure~\ref{fig:fairness}, for each value of $\gamma$ we show the
Pareto-optimal pairs, color-coded for the value of $\gamma$. Each value of $\gamma$
yields a set or cloud of undominated pairs that are usually fairly close to each other, and
as expected, as $\gamma$ is increased, these clouds generally move leftwards and upwards
(lower error and higher unfairness).

We anticipate that the practical use of our algorithm would, as we have done,
explore many values of $\gamma$ and then pick a model corresponding to a point
on the aggregated Pareto frontier across all $\gamma$, which represents the collection
of all undominated models and the overall error-unfairness trade-off.
This aggregate frontier is shown in the right panel of
Figure~\ref{fig:fairness}, and shows a relatively smooth menu of options,
ranging from error about 0.21 and no unfairness at one extreme, to
error about 0.12 and unfairness 0.025 at the other, and an appealing
assortment of intermediate trade-offs. Of course, in a real application
the selection of a particular point on the frontier should be made in
a domain-specific manner by the stakeholders or policymakers in question.

\paragraph*{Acknowledgements}
We thank Alekh Agarwal, Richard Berk, Miro Dud{\'\i}k, Akshay
Krishnamurthy, John Langford, Greg Ridgeway and Greg Yang for helpful
discussions and suggestions.


\bibliographystyle{plainnat}

\bibliography{./refs}

\appendix

\section{Chernoff-Hoeffding Bound}

We use the following concentration inequality.

\begin{theorem}[Real-vaued Additive Chernoff-Hoeffding
  Bound]\label{chernoff}
  Let $X_1, X_2, \ldots, X_m$ be i.i.d. random variables with
  $\Ex{}{X_i} = \mu$ and $a\leq X_i \leq b$ for all $i$. Then for
  every $\alpha  > 0$,
  \[
    \Pr\left[ \left| \frac{\sum_i X_i}{m} - \mu \right| \geq \alpha
    \right] \leq 2 \exp\left(\frac{-2\alpha^2 m }{(b-a)^2} \right)
  \]
\end{theorem}

\section{Generalization Bounds}
\label{app:generalization}

\begin{proof}[Proof of Theorems \ref{thm:SP-uc} and \ref{thm:FP-uc}]
We give a proof of Theorem \ref{thm:SP-uc}. The proof of Theorem \ref{thm:FP-uc} is identical, as false positive rates are just positive classification rates on the subset of the data for which $y = 0$.

Given a set of classifiers $\cH$ and protected groups $\cG$, define the following function class:
$$\mathcal{F}_{\cH, \cG} = \{f_{h,g}(x) \doteq h(x) \wedge g(x) : h \in \cH, g \in \cG\}$$
We can relate the VC-dimension of $\mathcal{F}_{\cH,\cG}$ to the VC-dimension of $\cH$ and $\cG$:
\begin{claim}
$$\mathrm{VCDIM}(\mathcal{F}_{\cH, \cG}) \leq \tilde{O}(\mathrm{VCDIM}(\cH) + \mathrm{VCDIM}(\cG))$$
\end{claim}
\begin{proof}
Let $S$ be a set of size $m$ shattered by $\mathcal{F}_{\cH, \cG}$. Let $\pi_{\mathcal{F}_{\cH, \cG}}(S)$ be the number of labelings of $S$ realized by elements of $\mathcal{F}_{\cH, \cG}$. By the definition of shattering, $\pi_{\mathcal{F}_{\cH, \cG}}(S) = 2^{m}$. Now for each labeling of $S$ by an element in $\mathcal{F}_{\cH, \cG}$, it is realized as $(f \wedge g)(S)$ for some $f \in \mathcal{F}, g \in \mathcal{G}$. But $(f \wedge g)(S) = f(S) \wedge g(S)$, and so it can be realized as the conjunction of a labeling of $S$ by an element of $\mathcal{F}$ and an element of $\mathcal{G}$. But since there are $\pi_{\mathcal{F}}(S)\pi_{\mathcal{G}}(S)$ such pairs of labelings, this immediately implies that $\pi_{\mathcal{F}_{\cH, \cG}}(S) \leq \pi_{\mathcal{F}}(S)\pi_{\mathcal{G}}(S)$. Now by the Sauer-Shelah Lemma (see e.g. \cite{KV94}), $\pi_{\mathcal{F}}(S) = O(m^{\mathrm{VCDIM}(\cH)}), \pi_{\cG}(S) = O(m^{\mathrm{VCDIM}(\cG)})$. Thus $\pi_{\mathcal{F}_{\cH, \cG}}(S) = 2^{m} \leq O(m^{\mathrm{VCDIM}(\cH) + \mathrm{VCDIM}(\cG)})$, which implies that $m = \tilde{O}(\mathrm{VCDIM}(\cH) + \mathrm{VCDIM}(\cG))$, as desired.

\end{proof}
This bound, together with a standard VC-Dimension based uniform convergence theorem (see e.g. \cite{KV94}) implies that with probability $1-\delta$, for every $f_{h,g} \in \mathcal{F}_{\cH, \cG}$:
$$\left|\E_{(X,y) \sim \mathcal{P}}[f_{h,g}(X)] - \E_{(X,y) \sim \cP_S}[f_{h,g}(X)] \right| \leq \tilde{O}\left(\sqrt{\frac{(\mathrm{VCDIM}(\cH) + \mathrm{VCDIM}(\cG))\log m + \log(1/\delta)}{m}} \right)$$
Note that the left hand side of the above inequality can be written as:
$$\left|\Pr_{(X,y) \sim \mathcal{P}}[h(X) = 1 | g(x) = 1] \cdot \Pr_{(X,y) \sim \mathcal{P}}[g(x) = 1] -  \Pr_{(X,y) \sim \cP_S}[h(X) = 1 | g(x) = 1] \cdot \Pr_{(X,y) \sim \cP_S}[g(x) = 1] \right|$$
This completes our proof.
\end{proof}

\section{Missing Proofs in \Cref{sec:learning}}

\approxmm*

\begin{proof}[Proof of \Cref{thm:approxmm}]
  Let $D^*$ be the optimal feasible solution for our constrained
  optimization problem.  Since $D^*$ is feasible, we know that
  $\cL(D^*, \hat \lambda) \leq err(D^*, \cP)$.

  We will first focus on the case where $\hat D$ is not a feasible
  solution, that is
  $$\max_{(g, \bullet) \in \cG(S) \times \{\pm\}} \Phi_\bullet(\hat D, g) > 0$$ Let
  $(\hat g, \hat \bullet)\in \argmax_{(g, \bullet)} \Phi_\bullet(\hat
  D, g)$ and let $\lambda'\in \Lambda$ be a vector with
  $(\lambda')_{\hat g}^{\hat \bullet} = C$ and all other coordinates
  zero. By \Cref{lem:dualbr}, we know that
  $\lambda' \in \argmax_{\lambda\in \Lambda} \cL(\hat D, \lambda)$.
  By the definition of a $\nu$-approximate minmax solution, we
  know that
  $\cL(\hat D, \hat \lambda) \geq \cL(\hat D, \lambda') - \nu$. This
  implies that
  \begin{equation}\label{batman}
    \cL(\hat D, \hat \lambda) \geq err(\hat D, \cP) + C \, \Phi_{\hat \bullet}(\hat D, \hat g) -
    \nu
  \end{equation}
  Note that $\cL(D^*, \hat \lambda) \leq err(D^*, \cP)$, and so
  \begin{equation}\label{superman}
    \cL(\hat D, \hat \lambda) \leq \min_{D\in \Delta_{\cH(S)}}\cL(D, \hat
    \lambda) + \nu \leq \cL(D^*, \hat
    \lambda) + \nu
\end{equation}
Combining \Cref{batman,superman}, we get
  \[
    err(\hat D, \cP) + C\, \Phi_{\hat \bullet}(\hat D, \hat g) \leq
    \cL(\hat D, \hat \lambda) + \nu \leq \cL(D^*, \hat \lambda) + 2\nu
    \leq err(D^*, \cP) + 2\nu
  \]
  Note that $C\, \Phi_{\hat \bullet}(\hat D, \hat g) \geq 0$, so we
  must have $err(\hat D, \cP) \leq err(D^*, \cP) +2 \nu = \OPT + 2\nu$.
  Furthermore, since $err(\hat D, \cP), err(D^*, \cP)\in [0, 1]$, we know
  \[
    C\, \Phi_{\hat \bullet}(\hat D, \hat g) \leq 1 + 2\nu,
  \]
  which implies that maximum constraint violation satisfies
  $\Phi_{\hat \bullet}(\hat D, \hat g) \leq (1+2\nu)/C$.
  By applying \Cref{cl:stuff}, we get
  \[
    \alpha_{FP}(g, \cP)\; \beta_{FP}(g, \hat D, \cP) \leq \gamma + \frac{1 + 2\nu}{C}.
  \]

  Now let us consider the case in which $\hat D$ is a feasible
  solution for the optimization problem. Then it follows that there is
  no constraint violation by $\hat D$ and
  $\max_\lambda \cL(\hat D, \lambda) = err(\hat D, \cP)$, and so
  \[
    err(\hat D, \cP) = \max_\lambda \cL(\hat D, \lambda) \leq \cL(\hat
    D, \hat \lambda) + \nu \leq \min_{D} \cL(D, \hat \lambda) + 2\nu
    \leq \cL(D^*, \hat \lambda) + 2\nu \leq err(D^*, \cP) + 2\nu
  \]
  Therefore, the stated bounds hold for both cases.
\end{proof}

\dualbr*

\begin{proof}[Proof of \Cref{lem:dualbr}]
Observe:
\begin{align*}
  \argmax_{\lambda\in \Lambda} \cL(\overline D, \lambda) &=
                                                           \argmax_{\lambda\in \Lambda} \Ex{h\sim \overline
                                                           D}{err(h, \cP)} + \sum_{g\in \cG(S)} \left( \lambda_g^+\,  \Phi_+(\overline D, g) +  \lambda_g^- \,  \Phi_-(\overline D,g)\right)\\
                                                         &=   \argmax_{\lambda\in \Lambda}  \sum_{g\in \cG} \left( \lambda_g^+\,  \Phi_+(\overline D, g) +  \lambda_g^- \,  \Phi_-(\overline D,g)\right)
\end{align*}
Note that this is a linear optimization problem over the non-negative
orthant of a scaling of the $\ell_1$ ball, and so has a solution at a
vertex, which corresponds to a single group $g \in \cG(S)$. Thus,
there is always a best response $\lambda'$ that puts all the weight
$C$ on the coordinate $(\lambda')_g^\bullet$ that maximizes
$\Phi_\bullet(\overline D, g)$.
\end{proof}

\learneregret*

\begin{proof}[Proof of \Cref{learner-regret}]
  To instantiate the regret bound in \Cref{thm:ftpl-regret}, we just
  need to provide a bound on the maximum absoluate value over the
  coordinates of the loss vector (the quantity $M$ in
  \Cref{thm:ftpl-regret}). For any $\lambda\in \Lambda$, the absolute
  value of the $i$-th coordinate of $\LC(\lambda)$ is bounded by:
  \begin{align*}
    &\frac{1}{n} + \left|\frac{1}{n}\sum_{g\in \cG(S)}
      (\lambda_g^+ - \lambda_g^-) \left(\Pr[g(x) = 1\mid y = 0] -  1\right)
      \, \mathbf{1}[g(x_i) = 1]\right|\\
    \leq & \frac{1}{n} + \frac{1}{n} \left(\sum_{g\in \cG(S)}\left| \lambda_g^+ - \lambda_g^- \right|\right) \max_{g\in \cG(S)} \left( \Pr[g(x) = 1\mid y=0] \mathbf{1}g(x_i)=1 \right)\\
    \leq &\frac{1}{n} + \frac{1}{n} \left(\sum_{g\in \cG(S)}\left| \lambda_g^+ \right| +\left| \lambda_g^- \right|\right) \leq \frac{1 + C}{n}
  \end{align*}
  Also note that the dimension of the optimization is the size of the
  dataset $n$. This means if we set
  $\eta = \frac{n}{(1+C)} \sqrt{\frac{1}{\sqrt{n} T}}$, the regret of
  the learner will then be bounded by $2 n^{1/4} {(1+C)} \sqrt{T}$.
\end{proof}

\samplingerr*

\begin{proof}[Proof of \Cref{lem:sampling-err}]
  Recall that for any distribution $D'$ over $\cH(S)$ the expected
  payoff function is defined as
  \begin{align*}
    \Ex{h\sim \hat D}{U(h, \lambda)} -  \Ex{h\sim D}{U(h, \lambda)} &= \Ex{h\sim \hat D}{err(h, \cP)} 
                                                                      + \Ex{h\sim \hat D}{\sum_{g\in \cG(S)}
                                                                      \left(\lambda_g^+\Phi_+(h, g) + \lambda_g^-\Phi_-(h, g) \right)}\\
&- \Ex{h\sim D}{err(h, \cP)} 
                                                                      + \Ex{h\sim D}{\sum_{g\in \cG(S)}
                                                                      \left(\lambda_g^+\Phi_+(h, g) + \lambda_g^-\Phi_-(h, g) \right)}
  \end{align*}
  By the triangle inequality, it suffices to show that with
  probability $(1 - \delta)$,
  $A = |\Ex{h\sim D}{err(h, \cP)} - \Ex{h\sim \hat D}{err(h, \cP)}|
  \leq \xi/2$ and for all $\lambda\in \Lambda$ and $g\in \cG(S)$,
  \[
    B = \left| \Ex{h\sim \hat D}{\sum_{g\in \cG(S)}
        \left(\lambda_g^+\Phi_+(h, g) + \lambda_g^-\Phi_-(h, g)
        \right)} - \Ex{h\sim D}{\sum_{g\in \cG(S)}
        \left(\lambda_g^+\Phi_+(h, g) + \lambda_g^-\Phi_-(h, g)
        \right)}\right| \leq \xi/2
  \]

    The first part follows directly from a simple application of the
    Chernoff-Hoeffding bound (\Cref{chernoff}): with probability
    $(1 - \delta/2)$, $A \leq \xi/2$, as long as
    $m\geq 2\ln(4/\delta)/\xi^2 $.

    To bound the second part, we first note that by H\"{o}lder's
    inequality, we have
    $$B\leq \|\lambda\|_1 \max_{(g, \bullet)\in \cG(S) \times
      \{\pm\}}|\Phi_\bullet(D, g) - \Phi_\bullet(\hat D, g)|$$ 

    Since for all $\lambda\in \Lambda$ we have $\|\lambda\|_1 \leq C$,
    it suffices to show that with probability $1 - \delta/2$,
    $|\Phi_\bullet(D, g) - \Phi_\bullet(\hat D, g)| \leq \xi/(2C)$
    holds for all $\bullet\in \{-, +\}$ and $g\in \cG(S)$. Note that
  \begin{align*}
    |\Phi_\bullet(D, g) - \Phi_\bullet(\hat D, g)| = \left|\left(\Ex{h\sim
    D}{ \FP(h)} - \Ex{h\sim \hat D}{\FP(h)} \right)\Pr[y=0 , g(x) =
    1] \right. \\+ \left. \left(\Ex{h\sim D}{\Pr[h(X) = 1, y= 0, g(x) = 1 ]} -
    \Ex{h\sim \hat D}{\Pr[h(X) = 1, y= 0, g(x) = 1 ]} \right)\right|
  \end{align*}
  We can rewrite the absolute value of first term:
  \begin{align*}
    &    \left|\left(\Ex{h\sim D}{ \FP(h)} - \Ex{h\sim \hat D}{\FP(h)}
      \right)\Pr[y=0 , g(x) = 1] \right|\\ = &\left|\left(\Ex{h\sim D}{ \Pr[h(X) = 1\mid y=0]} -
                                               \Ex{h\sim \hat D}{\Pr[h(X) = 1\mid y = 0]} \right)\Pr[g(x) =  1\mid y=0]\right|\\
    \leq  &\left|\left(\Ex{h\sim D}{ \Pr[h(X) = 1,  y=0]} -
            \Ex{h\sim \hat D}{\Pr[h(X) = 1, y = 0]} \right)\right|
  \end{align*}
  where the last inequality follows from
  $\Pr[g(x) = 1\mid y=0] \leq 1$.

  Note that
  $\Ex{h\sim \hat D}{\Pr[h(X) = 1, y= 0, g(x) = 1 ]} = \frac{1}{m}
  \sum_{j=1}^m \Pr[h^j(X) = 1, y= 0, g(x) = 1 ]$, which is an average
  of $m$ i.i.d. random variables with expectation
  $\Ex{h\sim D}{\Pr[h(X) = 1, y= 0, g(x) = 1 ]}$. By the
  Chernoff-Hoeffding bound (\Cref{chernoff}), we have
  \begin{equation}
    \Pr\left[ \left|\Ex{h\sim D}{ \Pr[h(X) = 1, y=0]} -
        \Ex{h\sim \hat D}{\Pr[h(X) = 1, y = 0]} \right| >
      \frac{\xi}{4C} \right] \leq 2\exp\left(-\frac{ \xi^2 m}{ 8 C^2}
    \right)\label{hp1}\end{equation}
  In the following, we will let $\delta_0 = 2\exp\left(-\frac{ \xi^2 m}{ 8 C^2}
  \right)$.  Similarly, we also have for each $g\in \cG(S)$,
  \begin{equation}
    \Pr\left[\left| \Ex{h\sim D}{\Pr[h(X) = 1, y= 0, g(x) = 1 ]} -
        \Ex{h\sim \hat D}{\Pr[h(X) = 1, y= 0, g(x) = 1 ]}\right| >
      \frac{\xi}{4C}\right] \leq \delta_0\label{hp2}
  \end{equation}

  By taking the union bound over \eqref{hp1} and \eqref{hp2} over all
  choices of $g\in \cG(S)$, we have with probability at least
  $(1 - \delta_0(1 + |\cG(S)|))$,
  \begin{equation} \left|\Ex{h\sim D}{ \Pr[h(X) = 1, y=0]} - \Ex{h\sim
        \hat D}{\Pr[h(X) = 1, y = 0]} \right| \leq \frac{\xi}{4C}
    \quad
\label{wonder}\end{equation}
{and,}
  \begin{equation}
    \left| \Ex{h\sim D}{\Pr[h(X) = 1, y= 0, g(x) = 1 ]} - \Ex{h\sim
        \hat D}{\Pr[h(X) = 1, y= 0, g(x) = 1 ]}\right| \leq
    \frac{\xi}{4C} \quad \mbox{for all } g \in \cG(S). \label{woman}
  \end{equation}
  Note that by Sauer's lemma (\Cref{lem:sauer}),
  $|\cG(S)|\leq O\left( n^{d_2} \right)$. Thus, there exists an
  absolute constant $c_0$ such that
  $m \geq c_0\frac{C^2\left(\ln(1/\delta) + d_2\ln(n)\right)}{\xi^2}$
  implies that failure probability above
  $\delta_0(1 + |\cG(S)|) \leq \delta/2$. We will assume $m$
  satisifies such a bound, and so the events of \eqref{wonder} and
  \eqref{woman} hold with probaility at least $(1 - \delta/2)$.  Then
  by the triangle inequality we have for all
  $(g, \bullet)\in \cG(S)\times\{\pm\}$,
  $|\Phi_\bullet(D, g) - \Phi_\bullet(\hat D, g)| \leq \xi/(2C)$,
  which implies that $B \leq \xi/2$. This completes the proof.
\end{proof}

\begin{claim}\label{cor:approx-br}
  Suppose there are two distributions $D$ and $\hat D$ over $\cH(S)$
  such that
  \[
    \max_{\lambda\in \Lambda} \left| \Ex{h\sim \hat D}{U(h, \lambda)}
      - \Ex{h\sim D}{U(h, \lambda)} \right| \leq \xi.
  \]
Let
  $$\hat{\lambda} \in \argmax_{\lambda'\in \Lambda} \Ex{h\sim \hat
    D}{U(h, \lambda')}$$ 
Then
  \[
    \max_{\lambda} \Ex{h\sim D}{U(h, \lambda)} -\xi \leq \Ex{h\sim
      D}{U(h, \hat \lambda)},
  \]
\end{claim}

\auditorbr*

\begin{proof}
  Let $\gamma_A^t$ be defined as
  \[
    \gamma_A^t = \max_{\lambda\in \Lambda} \left| \Ex{h\sim \hat
        D^t}{U(h, \lambda)} - \Ex{h\sim D^t}{U(h, \lambda)} \right|
  \]
  By instantiating \Cref{lem:sampling-err} and applying union bound
  across all $T$ steps, we know with probability at least
  $1 - \delta$, the following holds for all $t\in [T]$:
  \[
    {\gamma_A^t} \leq \sqrt{\frac{c_0 C^2 (\ln(T/\delta) + d_2
        \ln(n))}{m}}
  \]
  where $c_0$ is the absolute constant in \Cref{lem:sampling-err} and
  $d_2 = \VCD(\cG)$. 

  Note that by \Cref{cor:approx-br}, the Auditor is performing a
  $\gamma_A^t$-approximate best response at each round $t$. Then we
  can bound the Auditor's regret as follows:
  \begin{align*}
    \gamma_A = \frac{1}{T}\left[\max_{\lambda\in \Lambda} \sum_{t=
    1}^T \Ex{h\sim D^t}{U(h, \lambda)} - \sum_{t=1}^T \Ex{h\sim
    D^t}{U(h, \lambda^t)}\right] &\leq \frac{1}{T}\sum_{t=
                                   1}^T\left( \max_{\lambda\in \Lambda} \Ex{h\sim D^t}{U(h, \lambda)} - \Ex{h\sim
                                   D^t}{U(h, \lambda^t)} \right)  \\
                                 &\leq \max_T \gamma_A^t
  \end{align*}
  It follows that with probability $1 - \delta$, we have
\begin{align*}
  \gamma_A &\leq \sqrt{\frac{c_0 C^2 (\ln(T/\delta) + d_2
        \ln(n))}{m}}
\end{align*}
which completes the proof.
\end{proof}

\end{document}